\documentclass[accepted, mathfont=cm]{uai2021} % after acceptance, for a revised
\usepackage{graphicx}
\usepackage{booktabs}
\PassOptionsToPackage{hyphens}{url}\usepackage{hyperref}
\usepackage{amsmath}
\usepackage{empheq}
\usepackage{amsfonts}
\usepackage{amssymb}
\usepackage{amsthm}
\usepackage{mathtools}
\usepackage{amsopn}
\usepackage{pgfplots}
\pgfplotsset{compat=1.13}
\usepgfplotslibrary{groupplots}
\usepackage{siunitx}

\usepackage{algorithm}
\usepackage{algorithmic}

\usepackage{tikz}
\usepackage{tikzscale}

\usetikzlibrary{backgrounds}
\usetikzlibrary{intersections}
\usetikzlibrary{shapes}
\usetikzlibrary{plotmarks}
\usepgfplotslibrary{fillbetween}

\usetikzlibrary{positioning}

\usepackage[capitalize,nameinlink]{cleveref}
\crefname{section}{Sec.}{Secs.}
\crefname{appendix}{App.}{Apps.}

\usetikzlibrary{backgrounds}
\usetikzlibrary{intersections}
\usetikzlibrary{shapes}
\usetikzlibrary{plotmarks}
\usepgfplotslibrary{fillbetween}

\usetikzlibrary{positioning}

\usepackage{caption}

\usepackage{grffile}

\usepackage{microtype}
\usepackage{subcaption}

\newcommand{\Func}[2]{#1\!\left(#2\right)}
\newcommand{\FuncSq}[2]{#1\!\left[ #2 \right]}
\newcommand{\CondFunc}[3]{#1\!\left(#2 \, \middle\vert \, #3 \right)}

\newcommand{\Prob}[1]{\Func{p}{#1}}
\newcommand{\Cond}[2]{\CondFunc{p}{#1}{#2}}

\newcommand{\qFunc}[1]{\Func{q}{#1}}
\newcommand{\qFuncSpaced}[1]{\Func{q\,}{#1}}

\newcommand{\dNorm}[1]{\Func{\mathcal{N}}{#1}}

\newcommand{\dGP}[1]{\Func{\mathcal{GP}}{#1}}

\newcommand{\pdfNorm}[2]{\Func{\mathcal{N}}{#1; #2}}

\newcommand{\Expect}[2]{\FuncSq{\mathbb{E}_{#1}}{#2}}

\newcommand{\DivKL}[2]{\FuncSq{\mathcal{KL}}{#1\!\mid\!\mid\! #2}}

\newcommand{\tr}[1]{\Func{\text{tr}}{#1}}

\newcommand{\Ident}{\mathbf{I}}

\newcommand{\inv}[1]{#1^{-1}}
\newcommand{\invbr}[1]{\left(#1\right)^{-1}}

\newcommand{\bigO}[1]{\Func{\mathcal{O}}{#1}}

\newcommand{\cov}[1]{\Func{\text{cov}}{#1}}

\newcommand{\bigCI}{\mathrel{\text{\scalebox{1.07}{$\perp\mkern-10mu\perp$}}}}

\newcommand{\transposesymbol}{\top}
\newcommand{\transpose}[1]{#1^\transposesymbol}
\newcommand{\invtranspose}[1]{#1^{-\transposesymbol}}

\newcommand{\kron}{\otimes}

\newtheorem{theorem}{Theorem}[section]

\newtheorem{lemma}[theorem]{Lemma}

\newcommand{\meanfunc}{m}
\newcommand{\kernel}{\kappa}
\newcommand{\inputdomain}{\mathcal{X}}
\newcommand{\reals}{\mathbb{R}}

\newcommand{\meanvec}{\mathbf{m}}
\newcommand{\covmat}{\mathbf{C}}

\renewcommand{\vec}[1]{\mathbf{#1}}

\newcommand{\spacevar}{\vec{r}}
\newcommand{\timevar}{\tau}

\newcommand{\timeprime}{\timevar^\prime}
\newcommand{\spaceprime}{\spacevar^\prime}

\newcommand{\yobs}{\vec{y}}

\newcommand{\zvec}{\mathbf{z}}

\newcommand{\uobs}{\mathbf{u}}

\newcommand{\fobs}{\mathbf{f}}
\newcommand{\fnotu}{{f_{\neq \uobs}}}
\newcommand{\Amat}{\mathbf{A}}

\newcommand{\meanvecq}{\meanvec^{\textrm{q}}}
\newcommand{\covmatq}{\covmat^{\textrm{q}}}
\newcommand{\precq}{\Lambda^{\textrm{q}}}

\newcommand{\meanqu}{\meanvecq_\uobs}
\newcommand{\covqu}{\covmatq_\uobs}

\newcommand{\meanq}{\hat{\meanvec}}

\newcommand{\xinp}{\vec{x}}

\makeatletter

\newcommand{\ifSubfilesClassLoaded}{%
    \expandafter\ifx\csname ver@subfiles.cls\endcsname\relax
        \expandafter\@secondoftwo
    \else
        \expandafter\@firstoftwo
    \fi
}

\makeatother

\newcommand{\transition}{\mathbf{A}}

\newcommand{\transitionvar}{\mathbf{Q}}

\newcommand{\emission}{\mathbf{H}}
\newcommand{\emissionvar}{\mathbf{S}}

\newcommand{\Bmat}{\mathbf{B}}
\newcommand{\Vmat}{\mathbf{V}}

\newcommand{\elbo}{\mathcal{L}}

\newcommand{\inner}[3]{\langle #2, #3\rangle_{#1}}

\newcommand{\Mpertime}{M_\timevar}
\newcommand{\pseudoinp}{\vec{z}}

\newcommand{\fpred}{\fobs_\ast}

\newcommand{\faux}{\bar{f}}
\newcommand{\fobsaux}{\bar{\fobs}}
\newcommand{\uobsaux}{\bar{\uobs}}
\newcommand{\kernelaux}{\bar{\kernel}}

\newcommand{\precopt}{\hat{\Lambda}^\ast}
\newcommand{\prect}{\mathbf{G}}

\newcommand{\CondProbApprox}[2]{\CondFunc{\tilde{p}}{#1}{#2}}

\newcommand{\xvec}{\mathbf{x}}
\newcommand{\yvec}{\mathbf{y}}
\newcommand{\avec}{\mathbf{a}}
\newcommand{\Qmat}{\mathbf{Q}}
\newcommand{\Umat}{\mathbf{U}}
\newcommand{\dimx}{D_x}
\newcommand{\dimy}{D_y}

\newcommand{\meanx}{\meanvec_{\xvec}}
\newcommand{\covx}{\covmat_{\xvec}}

\newcommand{\Cmat}{\mathbf{C}}

\newcommand{\meanxpost}{\meanvec_{\xvec \mid \yvec}}
\newcommand{\covxpost}{\covmat_{\xvec \mid \yvec}}

\newcommand{\meany}{\meanvec_{\yvec}}
\newcommand{\covy}{\covmat_{\yvec}}

\newcommand{\covyapprox}{\tilde{\covmat}_{\yvec}}

\newcommand{\dimvar}{d}

\newcommand{\domainX}{\mathcal{X}}
\newcommand{\domainY}{\mathcal{Y}}

\newcommand{\kernelx}{\kernel_{\textrm{x}}}
\newcommand{\kernely}{\kernel_{\textrm{y}}}

\newcommand{\setA}{\Func{f}{\domainX_1, \domainY_1}}
\newcommand{\setB}{\Func{f}{\domainX_2, \domainY_2}}
\newcommand{\setC}{\Func{f}{\domainX_2, \domainY_1}}

\newcommand{\spacedomain}{\inputdomain}
\newcommand{\timedomain}{\reals}
\newcommand{\dimdomain}{\{1, ..., D\}}

\newcommand{\approxobs}{\yobs^{\textrm{q}}}
\newcommand{\approxprec}{\Lambda^{\textrm{q}}}

\newcommand{\psdmats}{\mathbb{S}_+}

\usepackage[british]{babel}

\usepackage{natbib} % has a nice set of citation styles and commands
\newlength{\figurewidth}
\newlength{\figureheight}

\tikzset{help lines/.style={color=blue!50, thin}}

\tikzset{observed/.style={
    circle,
    draw,
    fill=gray!25,
    minimum size=1cm
}}
\tikzset{latent/.style={
    circle,
    draw,
    minimum size=1cm
}}

\tikzset{help lines/.style={color=blue!50, thin}}

\tikzstyle{bluecircle}=[rectangle, draw=blue, fill=blue, minimum size=4mm]
\tikzstyle{redcircle}=[rectangle, draw=red, fill=white, minimum size=4mm, fill opacity=1, very thick]
\tikzstyle{blackcircle}=[circle, draw=black, fill=black, minimum size=1mm]

\title{    Combining Pseudo-Point and State Space Approximations \\
    for Sum-Separable Gaussian Processes}

\author[1]{\href{mailto:Will Tebbutt <wct23@cam.ac.uk>?Subject=Your UAI 2021 paper}{Will Tebbutt}{}} % Lead author
\author[2]{Arno Solin}
\author[1]{Richard E. Turner}

\affil[1]{%
    University of Cambridge, UK
}
\affil[2]{%
    Aalto University, Finland
}

\begin{document}
\maketitle

\begin{abstract}
    Gaussian processes (GPs) are important probabilistic tools for inference and learning in spatio-temporal modelling problems such as those in climate science and epidemiology. However, existing GP approximations do not simultaneously support large numbers of off-the-grid spatial data-points and long time-series which is a hallmark of many applications.
    Pseudo-point approximations, one of the gold-standard methods for scaling GPs to large data sets, are well suited for handling off-the-grid spatial data. However, they cannot handle long temporal observation horizons effectively reverting to cubic computational scaling in the time dimension. State space GP approximations are well suited to handling temporal data, if the temporal GP prior admits a Markov form, leading to linear complexity in the number of temporal observations, but have a cubic spatial cost and cannot handle off-the-grid spatial data.
    In this work we show that there is a simple and elegant way to combine pseudo-point methods with the state space GP approximation framework to get the best of both worlds. The approach hinges on a surprising conditional independence property which applies to space--time separable GPs. We demonstrate empirically that the combined approach is more scalable and applicable to a greater range of spatio-temporal problems than either method on its own.
\end{abstract}

\section{Introduction}

\begin{figure}[t]
  \centering
  \begin{tikzpicture}[outer sep=0,inner sep=0]
    % Apply some TikZ magic to make vectorized rounded corners
    \node[minimum width=\columnwidth,minimum height=.66\columnwidth,rounded corners=2mm,path picture={
      \node at (path picture bounding box.center){
        \includegraphics[width=\columnwidth,trim=12 35 95 35,clip]{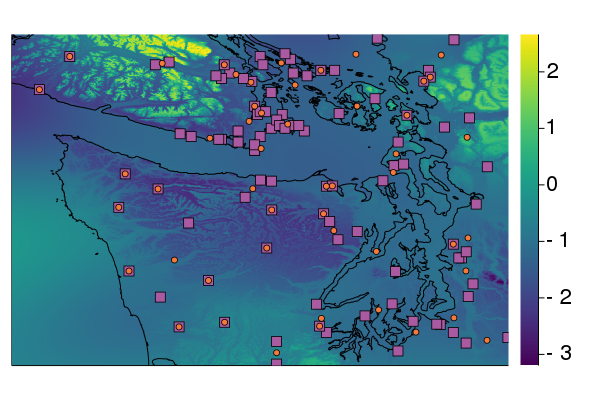}
      };}] at (0,0) {};
  \end{tikzpicture}
  \caption{Spatial slice of a large-scale spatio-temporal modelling problem: The posterior mean belief over max temperature (standardised scale, $-3$~\protect\includegraphics[width=1cm,height=2.5mm]{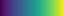}~$3$) on a day in early 2020 around Seattle and Vancouver. Pink squares are weather stations, orange dots are pseudo-points.}
  \label{fig:posterior-mean-weather-station}
\end{figure}

Large spatio-temporal data containing millions or billions of observations arise in various domains, such as climate science.
While Gaussian process (GP) models \citep{williams2006gaussian} can be useful models in such settings, the computational expense of exact inference is typically prohibitive, necessitating approximation.
This work combines two classes of approximations with complementary strengths and weaknesses to tackle spatio-temporal problems: pseudo-point \citep{quinonero2005unifying,bui2017unifying} and state-space \citep{sarkka2013spatiotemporal,sarkka2019applied} approximations.
\cref{fig:posterior-mean-weather-station} shows a single time-slice of a spatio-temporal model for daily maximum temperature, which extrapolates from fixed weather stations, constructed using this technique.

This work hinges on a conditional independence property possessed by separable GPs.
This property was identified by \citet{o1998markov}, and appears to have gone largely unnoticed within the GP community.
This property, in conjunction with the imposition of some structure on the pseudo-point locations, yields a collection of methods for approximate inference algorithm which scale linearly in time, the same as standard pseudo-point methods in space, and which can be implemented straightforwardly by utilising standard Kalman filtering-like algorithms.

In particular, we show
{\em (i)}~how O'Hagan's conditional independence property can be exploited to significantly accelerate the variational inference scheme of \citet{titsias2009variational} for GPs with separable and sum-separable kernels,
{\em (ii)}~how this can be straightforwardly combined with the Markov property exploited by state space approximations \citep{sarkka2019applied} to obtain an accurate approximate inference algorithm for sum-separable spatio-temporal GPs that scales linearly in time, and 
{\em (iii)}~how the earlier work of \citet{hartikainen2011sparse} on this topic is more closely related to the pseudo-point work of \citet{csato2002sparse} and \citet{snelson2005sparse} than previously realised.

\section{Sum-Separable Spatio-Temporal GPs}
\label{sec:sep-spat-temp-gps}

We call a GP \textit{separable across space and time} if its kernel is of the form
\begin{equation}
    \Func{\kernel}{(\spacevar, \timevar), (\spaceprime, \timeprime)} = \Func{\kernel^\spacevar}{\spacevar, \spaceprime} \Func{\kernel^\timevar}{\timevar, \timeprime} \label{eqn:separable-kernel}
\end{equation}
where $\spacevar, \spacevar^\prime \in \inputdomain$ are spatial inputs and $\timevar, \timevar^\prime \in \reals$ are temporal inputs.
We also call kernels such as $\kernel$ separable.
There is no particular restriction on what we define $\inputdomain$ to be -- it could be 3-dimensional Euclidean space in the literal sense, or it could be something else, such as a graph or the surface of a sphere.
Moreover, we place no restrictions on the form of $\kernel^\spacevar$, in particular we do not require it to be separable. Similarly, while the temporal inputs must be in $\reals$, it is irrelevant whether this dimension actually corresponds to time or to something else entirely.

This work considers a generalisation of separable GPs that we call \textit{sum-separable across space and time}, or simply \textit{sum-separable}.
We call a GP \textit{sum-separable} if it can be sampled by summing samples from a collection of independent separable GPs.
Specifically, let $f_p \sim \dGP{0, \kernel_p}$, $p=\{1, ..., P\}$, be a collection of $P$ independent separable GPs with kernels $\kernel_p$, and $f := \sum\nolimits_{p=1}^P f_p$, then $f$ is sum-separable.
$f$ has kernel
\begin{equation}
    \Func{\kernel}{(\spacevar, \timevar), (\spacevar^\prime, \timevar^\prime)} = \sum\nolimits_{p=1}^P \Func{\kernel_p}{(\spacevar, \timevar), (\spacevar^\prime, \timevar^\prime)},
\end{equation}
which is \emph{not} separable, meaning that sum-separable GPs such as $f$ are not generally separable.
In fact they are a much more expressive family of models, as they can represent processes which vary on multiple length scales in space and time.
Note that these are also distinct from additive GPs \citep{duvenaud2011additive} since each function depends on both space and time.

\section{Pseudo-Point Approximations}
\label{sec:vfe}

Pseudo-point approximations tackle the scaling problems of GPs by summarising a complete data set through a much smaller set of carefully-chosen uncertain pseudo-observations.

Consider a GP, $f \sim \dGP{\meanfunc, \kernel}$, of which $N$ observations $\yobs \in \reals^N$ are made at locations $\xinp \in \inputdomain^N$ through observation model $\Cond{\yobs}{\fobs} = \prod_{n=1}^N \Cond{\yobs_n}{\fobs_n}$, $\fobs_n := \Func{f}{\xinp_n}$.
The seminal work of \citet{titsias2009variational}, revisited by \citet{matthews2016sparse}, introduced the following approximation to the posterior distribution over $f$: \sloppy
\begin{equation}
    \qFunc{f} = \qFunc{\uobs} \Cond{\fnotu}{\uobs}, \label{eqn:qf}
\end{equation}
where $\uobs_m := \Func{f}{\pseudoinp_m}$ are the \emph{pseudo-points} for a collection of $M$ \emph{pseudo-inputs} $\pseudoinp_{1:M}$, and $\fnotu := f \setminus \uobs$ are all of the random variables in $f$ except those used as pseudo-points.
We assume that $\qFunc{\uobs}$ is Gaussian with mean $\meanqu$ and covariance matrix $\covqu$.
Subject to the constraint imposed in \cref{eqn:qf}, this family contains the optimal choice for $\Func{q}{\uobs}$ if each observation model $\Cond{\yobs_n}{\Func{f}{\xinp_n}}$ is Gaussian; moreover, a Gaussian form for $\qFunc{\uobs}$ is the de-facto standard choice when $\Cond{\yobs_n}{\Func{f}{\xinp_n}}$ is not Gaussian -- see e.g. \citet{hensman2013gaussian}.
This choice for $\qFunc{\uobs}$ yields the following approximate posterior predictive distribution at any collection of test points $\xinp_\ast$
\begin{align}
    \qFunc{\fpred} =&\, \pdfNorm{\fpred}{\covmat_{\fpred \uobs} \Lambda_{\uobs} \meanqu, \covmat}, \label{eqn:approx-post-pred} \\
    \covmat :=&\, \covmat_{\fpred} - \covmat_{\fpred \uobs} \Lambda_{\uobs} \covmat_{\uobs \fpred} + \covmat_{\fpred \uobs} \Lambda_{\uobs} \covqu \Lambda_{\uobs} \covmat_{\uobs \fpred}, \nonumber
\end{align}
where $\Lambda_{\uobs} := \inv{\covmat_{\uobs}}$ is the inverse of the covariance matrix between all pseudo-points, $\covmat_{\fpred \uobs}$ is the cross-covariance between the prediction points and pseudo-points under $f$, and $\meanvec_{\uobs}$ and $\meanvec_{\fpred}$ are the mean vectors at the pseudo-points and prediction points respectively.
For observation model
\begin{equation}
    \Cond{\yobs}{\fobs} = \pdfNorm{\yobs}{\fobs, \emissionvar} \label{eqn:obs-model}
\end{equation}
where $\emissionvar \in \reals^{N \times N}$ is a positive-definite diagonal matrix, it is possible to find the optimal $\qFunc{\uobs}$ in closed-form:
\begin{equation}
    \qFunc{\uobs} \propto \pdfNorm{\yobs}{\covmat_{\fobs \uobs} \Lambda_{\uobs} \uobs, \emissionvar} \pdfNorm{\uobs}{\mathbf{0}, \covmat_{\uobs}} \label{eqn:opt-approx-post}
\end{equation}
and the ELBO at this optimum is also closed-form:
\begin{multline}
    \elbo = \log \pdfNorm{\yobs}{\meanvec_\fobs, \covmat_{\fobs \uobs} \Lambda_{\uobs} \covmat_{\uobs \fobs} + \emissionvar}  \\
    - \frac{1}{2} \tr{ \inv{\emissionvar} ( \covmat_{\fobs} - \covmat_{\fobs \uobs} \Lambda_{\uobs} \covmat_{\uobs \fobs} ) },
\end{multline}
and is known as the \emph{saturated bound}.
It can be computed using only $\bigO{NM^2}$ operations using the matrix inversion and determinant lemmas

\paragraph{Related Models} The optimal $\qFunc{\uobs}$ coincides with the exact posterior distribution over $\uobs$ under an \textit{approximate model} with observation density $\pdfNorm{\yobs}{\covmat_{\fobs \uobs} \Lambda_{\uobs} \uobs, \emissionvar}$, and that the first term in $\elbo$ is the log marginal likelihood under this approximate model.
It is well known that this is precisely the approximation employed by \citet{seeger2003fast}, known as the \emph{Deterministic Training Conditional} (DTC). Despite their similarities, the DTC log marginal likelihood and the ELBO typically yield quite different kernel parameters and pseudo-inputs when optimised for -- while the pseudo-inputs $\pseudoinp_{1:M}$ are variational parameters in the variational approximation, and therefore not subject to overfitting (see section 2. of \citet{bui2017unifying}), they are model parameters in the DTC.
For this reason, the variational approximation is widely favoured over the DTC.

However, this close relationship between the variational approximation and the DTC is utilised in \cref{sec:exploiting-separability} to obtain algorithms which combine pseudo-point and state-space approximations in a manner which is both efficient, and easy to implement.

\paragraph{Benefits and Limitations} Pseudo-point approximations perform well when many more observations of a GP are made than are needed to accurately describe its posterior.
This is often the case for regression tasks where the inputs are sampled independently.
In this case the value of $M$ required to maintain an accurate approximation as $N$ increases generally seems not to grow too quickly---indeed \citet{burt2019rates} showed that if the inputs $\xinp_n$ are sampled i.i.d.\ from a Gaussian, then the value of $M$ required scales roughly logarithmically in $N$.
However, \citet{bui2014tree} noted that this is typically not the case for time series problems, where the interval in which the observations live typically grows linearly in $N$.
Indeed \citet{tobar2019band} showed that the number of the pseudo-points per unit time must not drop below a rate analogous to the Nyquist-Shannon rate if an accurate posterior approximation is to be maintained as $N$ grows.
Consequently the number of pseudo-points $M$ required to maintain a good approximation must grow linearly in $N$, so the cost of accurate approximate inference using pseudo-point methods is really $\bigO{N^3}$ in this case.

\section{State Space Approximations to Sum-Separable Spatio-Temporal GPs}
\label{sec:state-space}

Many time-series GPs can be augmented with additional latent dimensions in such a way that the marginal distribution over the original process is unchanged, but with the highly beneficial property that conditioning on all $D$ dimensions at any point in time renders past and future time points independent \citep{sarkka2019applied}.
This augmentation is exact for many GPs, in particular the popular half-integer Mat\'ern family, and a good approximation for others, such as those with exponentiated-quadratic kernels.
Consequently, for any collection of $T$ points in time, $\timevar_1 < \timevar_2 < ... < \timevar_T$, the augmented GP forms a $D$-dimensional Gauss-Markov chain, whose transition dynamics are a function of the kernel of the GP.
This means that standard algorithms (similar to Kalman filtering) can be utilised to perform inference under Gaussian likelihoods, thus achieving linear scaling in $T$.
This technique can be extended to separable and sum-separable spatio-temporal GPs for rectilinear grids of inputs, the details of which are as follows.

\paragraph{Separable GPs}
Let $\faux$ be such an augmentation of $f$ such that the distribution over $\Func{\faux}{\timevar, \spacevar, 1}$ is approximately equal to that of $\Func{f}{\timevar, \spacevar}$, and conditioning on all latent dimensions renders $\faux$ Markov in $\timevar$.
% This approximation can often be made tight, and achieves equality in various useful cases (e.g., the Mat\'ern family of GPs).
$\faux$ is specified implicitly through a linear stochastic differential equation, meaning that inference under Gaussian observations can be performed efficiently via filtering / smoothing in a Linear-Gaussian State Space Model (LGSSM).
Let $\fobsaux_t$ be the collection of random variables in $\faux$ at inputs given by the Cartesian product between the singleton $\{t\}$, $N_T$ arbitrary locations in space $\spacevar_{1:N_T}$, and all of the latent dimensions $\{1, \ldots, D\}$.
Let the kernel of $f$ be separable: $\Func{\kernel}{(\spacevar, \timevar), (\spaceprime, \timeprime)} = \Func{\kernel^\spacevar}{\spacevar, \spaceprime} \Func{\kernel^\timevar}{\timevar, \timeprime}$.
Any collection of finite dimensional marginals $\fobsaux := \fobsaux_{1:T}$, each using the same $\spacevar_{1:N_T}$, form an LGSSM with $N_T D$-dimensional state with dynamics
\begin{align}
    \fobsaux_{t} \mid \fobsaux_{t-1} \sim&\, \dNorm{ [\Ident_{N_T} \kron \transition_t] \, \fobsaux_{t-1}, \covmat_\fobs^\spacevar \kron \transitionvar_t} \label{eqn:sep-model} \\
    \emission_{ab} :=&\, \Ident_a \kron \begin{bmatrix} 1 & \mathbf{0}_{1 \times b-1} \end{bmatrix} \\
    \fobs_{t} =&\, \emission_{N_T D} \, \fobsaux_{t}, \label{eqn:fnt} \\
    \yobs_{t} \mid \fobs_{t} \sim&\, \dNorm{\fobs_{t}, \emissionvar_{t}} \label{eqn:ynt}
\end{align}
where $\kron$ denotes the Kronecker product, $\transition_t \in \reals^{D \times D}$ and $\transitionvar_t \in \reals^{D \times D}$ are functions of $\kernel^\timevar$, $\transitionvar_t$ is positive definite, $\covmat_\fobs^\spacevar$ is the covariance matrix associated with $\kernel^\spacevar$ and $\spacevar_{1:N_T}$, $\mathbf{0}_{p \times q}$ is a $p \times q$ matrix of zeros, $\yobs_t$ is the block of $\yobs$ containing the observations at the $t^{th}$ time, and the diagonal matrix $\emissionvar_{t}$ is the on-diagonal block of $\emissionvar$ corresponding to $\yobs_{t}$.
See \citet{solin2016stochastic} for further details about $\transition_t$ and $\transitionvar_t$.

\paragraph{Sum-Separable GPs}
Let $f$ be the sum-separable GP given by summing over $f_p \sim \dGP{0, \kernel_p}$.
A state space approximation to $f$ is obtained by constructing a $D_p$-dimensional state space approximation for each $f_p$, the finite dimensional marginals of which form an LGSSM
\begin{align}
    \fobsaux^{p}_{t} \mid \fobsaux^{p}_{t-1} \sim&\, \dNorm{ [\Ident_{N_T} \kron \transition_t^{p}] \, \fobsaux^{p}_{t-1}, [\covmat_{\fobs}^{\spacevar,p} \kron \transitionvar^{p}_t ] } \label{eqn:sum-sep-model} \\
    \fobs_t =&\, \sum_{p=1}^P \emission_{N_T D_p} \fobsaux_t^p
\end{align}
where $\transition_t^p$, $\transitionvar_t^p$, and $\covmat_{\fobs}^{\spacevar, p}$ are defined in the same way as above for each $f_p$, and $\yobs_{t} \mid \fobs_t$ is again given by \cref{eqn:ynt}.
This LGSSM has $N_T \sum_{p=1}^P D_p$ latent dimensions, increasing the time and memory needed to perform inference when compared to a separable model, and is the price of a more flexible model.

\paragraph{Benefits and Limitations}
While this formulation truly scales linearly in $T$ it has two clear limitations,
{\em (i)} all locations of observations must lie on a rectilinear time-space grid if any computational gains are to be achieved; and
{\em (ii)} inference scales cubically in $N_T$, meaning that inference is rendered infeasible by time or memory constraints if a large number of spatial locations are observed.

\section{Exploiting Separability to Obtain the Best of Both Worlds}
\label{sec:exploiting-separability}

We now turn to the main contribution of this work: combining the pseudo-point and state space approximations.
The result is an approximation which is applicable to any sum-separable GP whose time kernels can be approximated by a linear SDE.
We do this simply by constructing a variational pseudo-point approximation of the state space approximation to the original process.
In cases where the state space approximation is exact, this is similar to constructing an inter-domain pseudo-point approximation \citep{lazaro2009inter} to the original process, where some of the pseudo-points are placed in auxiliary dimensions.

In this section we show that by constraining the pseudo-inputs, approximate inference becomes linear in time.

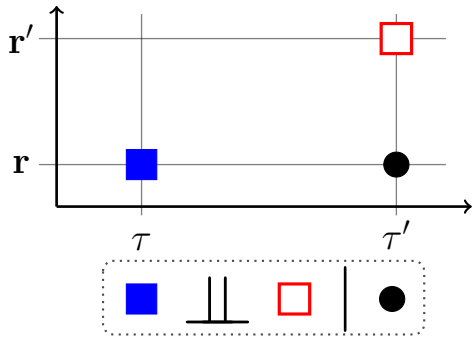
\begin{figure}[!b]
    % \centering
    \begin{tikzpicture}

        \def\mywidth{2.2in}
        \def\myheight{1.1in}

        \node (lhstopcorner) at (0, \myheight) {};
        \node (rhsbottomcorner) at (\mywidth, 0) {};

        \draw[->, line width=1] (0, 0) -- (lhstopcorner);
        \draw[->, line width=1] (0, 0) -- (rhsbottomcorner);

        % Vertical line above \tau
        \draw[-, line width=0.5, opacity=0.5] (\mywidth * 1 / 5, -\myheight * 1 / 24) -- (\mywidth * 1 / 5, \myheight * 11 / 12);
        \node at (\mywidth * 1 / 5, -\mywidth * 2 / 24) [] {\Large $\timevar$};

        % Vertical line above \tau\prime
        \draw[-, line width=0.5, opacity=0.5] (\mywidth * 4 / 5, -\myheight * 1 / 24) -- (\mywidth * 4 / 5, \myheight * 11 / 12);
        \node at (\mywidth * 4 / 5, -\mywidth * 1.6 / 24) [] {\Large $\timeprime$};

        % Horizontal line from r.
        \draw[-, line width=0.5, opacity=0.5] (-\mywidth * 1 / 24, \myheight * 1 / 5) -- (\mywidth * 11 / 12, \myheight * 1 / 5);
        \node at (-\mywidth * 2 / 24, \myheight * 1 / 5) [] {\Large $\spacevar$};

        % Horizontal line from r\prime
        \draw[-, line width=0.5, opacity=0.5] (-\mywidth * 1 / 24, \myheight * 4 / 5) -- (\mywidth * 11 / 12, \myheight * 4 / 5);
        \node at (-\mywidth * 2 / 24, \myheight * 4 / 5) [] {\Large $\spaceprime$};

        % Positions of the dots.
        \node at (\mywidth / 5, \myheight / 5) [bluecircle] {};
        \node at (\mywidth * 4 / 5, \myheight * 4 / 5) [redcircle] {};
        \node at (\mywidth * 4 / 5, \myheight * 1 / 5) [blackcircle] {};

        % Conditional independence statement.
        \node (foo) at (\mywidth * 1 / 5, -\mywidth * 2.2 / 10) [bluecircle] {};
        \node (bar) [right=0.2 of foo] {\Huge $\bigCI$};
        \node (baz) [right=0.2 of bar] [redcircle] {};
        \node (bleh) [right=0.2 of baz] {\Huge $\mid$};
        \node (hmm) [right=0.2 of bleh] [blackcircle] {};
        
        % Additional labels
        %\node [rotate=90] at (-\mywidth * 1 / 6, \myheight / 2) {Space};
        %\node [rotate=0] at (\mywidth * 1 / 2, -\myheight * 3 / 10) {Time};
        
        % For manual control of the location of the picture relative to the left border. Making this node further
        % to the left pushes the figure further to the right. The automatic \centering command wasn't quite doing
        % what I needed.
        \node (leftmost) [left=2 of foo] {};
        % \node (rightmost) [right=0.1 of hmm] {};

        % \draw (\mywidth * 1 / 5, \mywidth * 4 / 5) node {};
        
        % Draw rectangle
        \draw[draw=black!70,thick,dotted, rounded corners] ($(foo.north west)+(-0.3,0.3)$)  rectangle ($(hmm.south east)+(0.3,-0.35)$);

    \end{tikzpicture}
    \vspace*{-5pt}
    \caption{\label{fig:basic-cond-indep}Depiction of the conditional independence property in \cref{eqn:cond-indep-prop}. The blue square is $\Func{f}{\spacevar, \timevar}$, the red square is $\Func{f}{\spacevar^\prime, \timevar^\prime}$, and the black circle is $\Func{f}{\spacevar, \timevar^\prime}$.}
\end{figure}

\subsection{The Conditional Independence Structure of Separable GPs}

\newcommand{\spaceset}{\mathcal{R}}
\newcommand{\timeset}{\mathcal{T}}

\citet{o1998markov} showed that a separable GP $\Func{f}{\spacevar, \timevar}$ has the following conditional independence properties:
\begin{align}
    \Func{f}{\spacevar, \timevar} \bigCI \Func{f}{\spacevar^\prime, \timevar^\prime} \mid \Func{f}{\spacevar, \timevar^\prime}, \label{eqn:cond-indep-prop} \\
    \Func{f}{\spacevar, \timevar} \bigCI \Func{f}{\spacevar^\prime, \timevar^\prime} \mid \Func{f}{\spacevar^\prime, \timevar}.
\end{align}
These are explained graphically in \cref{fig:basic-cond-indep}.
It is straightforward to show (see \cref{sec:cond-indep-collections}) that this property extends to collections of random variables in $f$:
\begin{align}
    &\Func{f}{\spaceset, \timeset} \bigCI \Func{f}{\spaceset^\prime, \timeset^\prime} \mid \Func{f}{\spaceset, \timeset^\prime} \quad  \text{where} \label{eqn:cond-indep-collection} \\
    &\quad \Func{f}{\spaceset, \timeset} := \{ \Func{f}{\spacevar, \timevar} \mid \spacevar \in \mathcal{R}, \timevar \in \mathcal{T} \}  \nonumber \\
    &\quad \Func{f}{\spaceset^\prime, \timeset^\prime} := \{ \Func{f}{\spacevar, \timevar^\prime} \mid \spacevar \in \mathcal{R}^\prime \} \nonumber \\
    &\quad \Func{f}{\spaceset, \timeset^\prime} := \{ \Func{f}{\spacevar, \timevar^\prime} \mid \spacevar \in \mathcal{R} \} \nonumber
\end{align}
where $\mathcal{R}$ and $\mathcal{R}^\prime$ are sets of points in space, $\mathcal{T}$ is a set of points through time, and $\timevar^\prime \in \mathcal{T}$.
This conditional independence property is depicted in \cref{fig:cond-indep-collection}, and it is this second property that sits at the core of the approximation introduced in the next section.

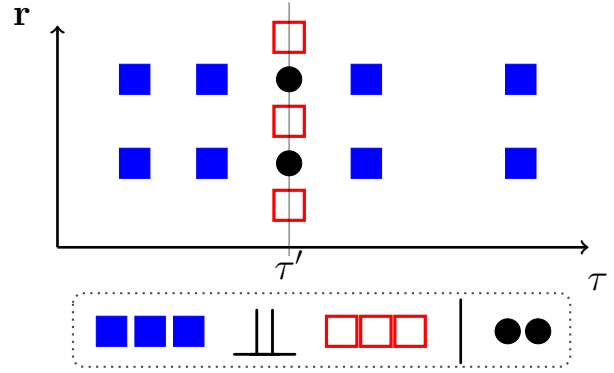
\begin{figure}[!t]
    \centering
    \begin{tikzpicture}

        \def\myheight{1.1in}
        \def\mywidth{4in}

        \centering

        % Specify corners.
        \node (lhstopcorner) at (0, \myheight * 1.1) {};
        \node (rhsbottomcorner) at (0.7 * \mywidth, 0) {};

        % Draw axes.
        \draw[->, line width=1] (0, 0) -- (lhstopcorner);
        \draw[->, line width=1] (0, 0) -- (rhsbottomcorner);

        % Draw points through time.

        % Draw blue circles bottom row.
        \foreach \t in {1, 2, 4, 6}
            \node at (\mywidth * \t / 10, \myheight * 2 / 5) [bluecircle] {};

        % Draw blue circles (top row).
        \foreach \t in {1, 2, 4, 6}
            \node at (\mywidth * \t / 10, \myheight * 4 / 5) [bluecircle] {};

        % Draw black circles
        \node at (\mywidth * 3 / 10, \myheight * 4 / 5) [blackcircle] {};
        \node at (\mywidth * 3 / 10, \myheight * 2 / 5) [blackcircle] {};

        % Draw red circles.
        \foreach \r in {1, 3, 5}
            \node at (\mywidth * 3 / 10, \myheight * \r / 5) [redcircle] {};

        % Vertical line above \tau\prime
        \draw[-, line width=0.5, opacity=0.5] (\mywidth * 3 / 10, -\myheight * 1 / 24) -- (\mywidth * 3 / 10, \myheight * 14 / 12);
        \node at (\mywidth * 3 / 10, -\myheight * 2 / 24) [] {\Large $\timeprime$};

        % Conditional independence statement. This is quite messy, but I couldn't figure out a simpler way to do it
        % that didn't give inconsistent spacing.
        \node (foo) at (\mywidth * 0.35 / 5, -\myheight * 4 / 10) [bluecircle] {};
        \node (foo1) at (0.2in + \mywidth * 0.35 / 5, -\myheight * 4 / 10) [bluecircle] {};
        \node (foo2) at (0.4in + \mywidth * 0.35 / 5, -\myheight * 4 / 10) [bluecircle] {};
        \node (bar) [right=0.2 of foo2] {\Huge $\bigCI$};
        \node (baz) [right=0.2 of bar] [redcircle] {};
        \node (baz1) [right=0.65 of bar] [redcircle] {};
        \node (baz2) [right=1.1 of bar] [redcircle] {};
        \node (bleh) [right=0.2 of baz2] {\Huge $\mid$};
        \node (hmm) [right=0.2 of bleh] [blackcircle] {};
        \node (hmm1) [right=0.6 of bleh] [blackcircle] {};

        \node (leftmost) [left=0.1 of foo] {};
        \node (rightmost) [right=0.1 of hmm] {};

        \draw (\mywidth * 1 / 5, \myheight * 4 / 5) node {};

        % Axis labels.
        \node (timevar) [below=0.1 of rhsbottomcorner] {\Large $\timevar$};
        \node (spacevar) [left=0.1 of lhstopcorner] {\Large $\spacevar$};

        % Draw rectangle
        \draw[draw=black!70,thick,dotted, rounded corners] ($(foo.north west)+(-0.3,0.3)$)  rectangle ($(hmm1.south east)+(0.3,-0.35)$);

    \end{tikzpicture}

    \vspace*{-5pt}

    \caption{\label{fig:cond-indep-collection}Depiction of the conditional independence property in \cref{eqn:cond-indep-collection}. The blue squares are $\Func{f}{\spaceset, \timeset}$, the red squares are $\Func{f}{\spaceset^\prime, \timeset^\prime}$, and the black circles are $\Func{f}{\spaceset, \timeset^\prime}$.}
\end{figure}

\subsection{Combining the Approximations}

We now combine the pseudo-point and state space approximations, and show how a temporal conditional independence property means that the optimal approximate posterior is Markov.
This in turn leads to a closed-form expression for the optimum under Gaussian observation models and the existence of a simplified LGSSM in which exact inference yields optimal approximate inference in the original model.

\paragraph{Pseudo-Point Approximation of State Space Augmentation}
We perform approximate inference in a separable GP $f$ with the kernel in \cref{eqn:separable-kernel} by applying the standard variational pseudo-point approximation (\cref{sec:vfe}) to its state space augmentation (\cref{sec:state-space}) $\faux$:
\begin{equation}
    \qFunc{\faux} := \qFunc{\uobsaux} \Cond{\faux_{\neq \uobsaux}}{\uobsaux}, \quad \qFunc{\uobsaux} = \pdfNorm{\uobsaux}{\meanvecq_{\uobsaux}, \covmatq_{\uobsaux}}, \nonumber
\end{equation}
where the pseudo-points $\uobsaux = \uobsaux_{1:T}$ form a rectilinear grid of points in time, space, and \emph{all} of the latent dimensions with the same structure as $\fobsaux$ in \cref{sec:state-space}, but replacing $\spacevar_{1:N_T}$ with a collection of $\Mpertime$ spatial pseudo-inputs, $\zvec_{1:\Mpertime}$, for a total of $T \Mpertime D$ pseudo-points. $\Prob{\uobsaux}$ is therefore Markov-through-time with conditional distributions
\begin{align}
    \uobsaux_t \mid \uobsaux_{t-1} \sim&\, \dNorm{[\Ident_{\Mpertime} \kron \transition_t] \uobsaux_{t-1}, \covmat_{\uobs}^\spacevar \kron \transitionvar_t}, \\
    \uobs_t :=&\, \emission_{\Mpertime D} \uobsaux_t.
\end{align}
where $\covmat_{\uobs}^\spacevar$ is the covariance matrix associated with $\kernel^\spacevar$ and $\zvec_{1:\Mpertime}$.
Note the resemblance to \cref{eqn:sep-model}.
No constraint is placed on the location of the pseudo-points in space, only that they must remain at the same place for all time points.

Crucially, we now relax the assumption that the inputs associated with $\fobs$ must form a rectilinear grid.
Instead, it is necessary only to require that each observation is made at one of the $T$ times at which we have placed pseudo-points.
We denote the number of observations at time $t$ by $N_t$, and continue to denote by $\fobs_t$ the set of observations at time $t$.

\paragraph{Exploiting Conditional Independence} Due to \citet{o1998markov}'s conditional independence property, $\Cond{\fobsaux_t}{\uobsaux} = \Cond{\fobsaux_t}{\uobs_t}$; see \cref{sec:cond-indep-results} for details. Consequently, the reconstruction terms in the ELBO depend only on $\uobs_t$ as opposed to the entirety of $\uobsaux$:
\begin{align}
    \elbo =&\, \sum_{t=1}^T r_t - \DivKL{\qFunc{\uobsaux}}{\Prob{\uobsaux}} \label{eqn:state-space-elbo}, \\
     r_t :=&\, \Expect{\qFuncSpaced{\uobs_t}}{\Expect{\CondFunc{p\,}{\fobs_{t}}{\uobs_t}}{\log \Cond{\yobs_{t}}{\fobs_{t}}}} \nonumber
\end{align}
This property alone yields substantial computational savings -- only the covariance between $\uobs_t$ and $\fobs_t$ need be computed, as opposed to all of $\uobsaux$ and $\fobs_t$.
Moreover, this means that
\begin{align}
    \covmat_{\fobs \uobsaux} \Lambda_{\uobsaux} =&\, \begin{bsmallmatrix}
        \Bmat_1 & & \mathbf{0} \\
        & \ddots & \\
        \mathbf{0} & & \Bmat_T
    \end{bsmallmatrix}, \,\, \Bmat_t := \covmat_{\fobs_t \uobs_t} \Lambda_{\uobs_t} \emission_{\Mpertime D}. \label{eqn:block-diag-cond}
\end{align}

\paragraph{The Optimal Approximate Posterior is Markov}

As an immediate consequence of \cref{eqn:state-space-elbo}, and by the same argument as that made by \citet{seeger1999bayesian}, highlighted by \citet{opper2009variational}, the optimal approximate posterior precision satisfies
\begin{align}
    \precq_{\uobsaux} = \Lambda_{\uobsaux} + \begin{bsmallmatrix}
    \prect_1 & & \mathbf{0} \\
    & \ddots & \\
    \mathbf{0} & & \prect_T
    \end{bsmallmatrix}, \prect_t := -2 \nabla_{\covmatq_{t}} r_t. \label{eqn:optimal-approx-post-precision}
\end{align}
where $\precq_{\uobsaux} := [\covmatq_{\uobsaux}]^{-1}$, and $\covmatq_{t}$ is the $t^{th}$ block on the diagonal of $\covmatq_{\uobsaux}$.
Recall that the precision matrix of a Gauss-Markov model is block tridiagonal (see e.g. \citet{grigorievskiy2017parallelizable}), so $\Lambda_{\uobsaux}$ is block tridiagonal.
Further, the exact posterior precision of an LGSSM with a Gaussian observation model is given by the sum of this block tridiagonal precision matrix and a block-diagonal matrix with the same block size.
$\precq_{\uobsaux}$ has precisely this form, so the optimal approximate posterior over $\uobsaux$ must be a Gauss-Markov chain.

\paragraph{Approximate Inference via Exact Inference in an Approximate Model}
The above is equivalent to the optimal approximate posterior having density proportional to
\begin{align}
    \qFunc{\uobsaux} \propto \prod_{t=1}^T \Cond{\uobsaux_{t}}{\uobsaux_{t-1}} \pdfNorm{\approxobs_t}{\uobsaux_t, \prect_t^{-1}},
\end{align}
where $\approxobs_1, ..., \approxobs_T$ are a collection of $T$ surrogate observations, detailed in \cref{sec:optimal-approx-posterior-structure}.
Thus the optimal $\qFunc{\uobsaux}$ is given by exact inference in an LGSSM.
Moreover, \citet{ashman2020sparse} (App.~A) show that $\prect_t$ can be written as a sum of $N_t$ rank-1 matrices.

\paragraph{Solution for Gaussian Observation Models}
Under a Gaussian observation model, the optimal approximate posterior is given by the exact posterior under the DTC observation model, as discussed in section \cref{sec:vfe}.
\cref{eqn:block-diag-cond} means that the DTC observation model can be written as
\begin{equation}
    \pdfNorm{\yobs}{\covmat_{\fobs \uobsaux} \Lambda_{\uobsaux} \uobsaux, \emissionvar} = \prod_{t=1}^T \pdfNorm{\yobs_t}{\Bmat_t \uobsaux_t, \emissionvar_t}. \label{eqn:lgssm-separable}
\end{equation}
In conjunction with $\Prob{\uobsaux}$, this yields the required LGSSM.

This LGSSM can be exploited both to perform approximate inference and compute the saturated bound in linear time, repurposing existing code -- see \cref{sec:computing-stuff}.
This LGSSM also makes it clear, for example, how to employ the parallelised inference procedures proposed by \citet{sarkka2020temporal} and \citet{loper2020general} within this approximation.

\paragraph{Sum-Separable Models}
Extending this approximation to sum-separable processes is similar to the standard state space approximation.
The resulting LGSSM is
\begin{align}
    \uobsaux_{t}^p \mid \uobsaux_{t-1}^p \sim&\, \dNorm{[ \Ident_{\Mpertime} \kron \Amat_t^p ] \, \uobsaux_{t-1}^p, [ \covmat_{\uobs}^{\spacevar, p} \kron \Qmat_t^p] } \label{eqn:lgssm-sum-separable} \\
    \Cond{\yobs_t}{\uobsaux_t} =&\, \pdfNorm{\yobs_t}{{\textstyle \smash{\sum_{p=1}^P}} \Bmat_t^p \uobsaux_t^p, \emissionvar_t}. \nonumber \\
    \Bmat_t^p :=&\, \covmat_{\fobs_t^p \uobs_t^p} \Lambda_{\uobs_t^p} \emission_{\Mpertime D_p}. \nonumber
\end{align}
Note the resemblance to \cref{eqn:sum-sep-model}.

\paragraph{Efficient Inference in the Conditionals} The structure present in each $\Bmat_t^p$ can be used to accelerate inference.
In particular note that $\emission_{\Mpertime D_p}$ has size $\Mpertime \times D\Mpertime$ while $\covmat_{\fobs_{t} \uobs_{t}}^p \Lambda_{\uobs_t}^p$ is $N_t \times \Mpertime$.
Certainly $\Mpertime \leq D\Mpertime$ and typically $\Mpertime < N$, so this linear transformation forms a bottleneck.
\cref{sec:efficient-inference-in-conditional} explores this property, and shows how to exploit it to accelerate inference.

\paragraph{Computational Complexity} The total number of flops required to compute the saturated ELBO is $T (D \Mpertime)^3 + D^3 \Mpertime^2 + \Mpertime^2 \sum_{t=1}^T N_t$ to leading order.
This is a great deal fewer when $T$ is large than the $M^3 + M^2N = \Mpertime^3 T^3 + \Mpertime^2 T^2 N $ required if the bound is computed naively.
Similar improvements are achieved when making posterior predictions.

\paragraph{Utilising Other Pseudo-Point Approximations} The conditional independence property exploited to develop the variational approximation in this section also shines new light on the work of \citet{hartikainen2011sparse}.
In the specific case of their equation 5, in which the observation model is (adopting their notation) $\Cond{\yobs_k}{\vec{x}_k} = \pdfNorm{\yobs_k}{[ \Ident_N \kron \emission ] \vec{x}_k, \emissionvar_t}$, they perform approximate inference in $\Prob{\uobsaux}$ using the modified observation model
\begin{align}
    \CondProbApprox{\yobs_t}{\uobsaux_t} :=&\, \pdfNorm{\yobs_t}{\covmat_{\fobs_t \uobsaux_t} \Lambda_{\uobsaux_t} \uobsaux_t, [\covyapprox]_t}, \nonumber \\
    % [\meanyapprox]_t :=&\, \covmat_{\fobs_t \uobsaux_t} \Lambda_{\uobsaux_t} \uobsaux_t, \nonumber \\
    [\covyapprox]_t :=&\, \Func{\text{diag}}{\covmat_{\fobs_t} - \covmat_{\fobs_t \uobsaux_t} \Lambda_{\uobsaux_t} \covmat_{\uobsaux_t \fobs_t}} + \emissionvar_t \nonumber
\end{align}
which is inspired by the well-known FITC \citep{csato2002sparse, snelson2005sparse} approximation.
However, due to \citet{o1998markov}'s conditional independence property, this is equivalent to
\begin{align}
    \CondProbApprox{\yobs}{\uobsaux} :=&\, \pdfNorm{\yobs}{\covmat_{\fobs \uobsaux} \Lambda_{\uobsaux} \uobsaux, \covyapprox}, \nonumber \\
    % \meanyapprox :=&\, \covmat_{\fobs \uobsaux} \Lambda_{\uobsaux} \uobsaux, \nonumber \\
    \covyapprox :=&\, \Func{\text{diag}}{\covmat_{\fobs} - \covmat_{\fobs \uobsaux} \Lambda_{\uobsaux} \covmat_{\uobsaux \fobs}} + \emissionvar. \nonumber
\end{align}
While \citet{hartikainen2011sparse} did not actually consider the Gaussian observation model in their work, it is clear from the above that they would have utilised \emph{exactly} the FITC approximation applied to $\faux$ had they done so.

\citet{bui2017unifying} showed that both FITC and VFE can be viewed as edge cases of the Power EP algorithm introduced by \citet{minka2004power}. Consequently the equivalent approximate model generalised both that of FITC and VFE -- only $\covyapprox$ is changed from FITC: let $\alpha \in [0, 1]$, then
\begin{equation}
    \covyapprox := \alpha \, \Func{\text{diag}}{\covmat_{\fobs} - \covmat_{\fobs \uobsaux} \Lambda_{\uobsaux} \covmat_{\uobsaux \fobs}} + \emissionvar. \nonumber
\end{equation}

In short, most standard pseudo-point approximations can be straightforwardly combined with state space approximations for sum-separable spatio-temporal GPs in the manner that we propose due to the conditional independence property.

\begin{figure}[!t]
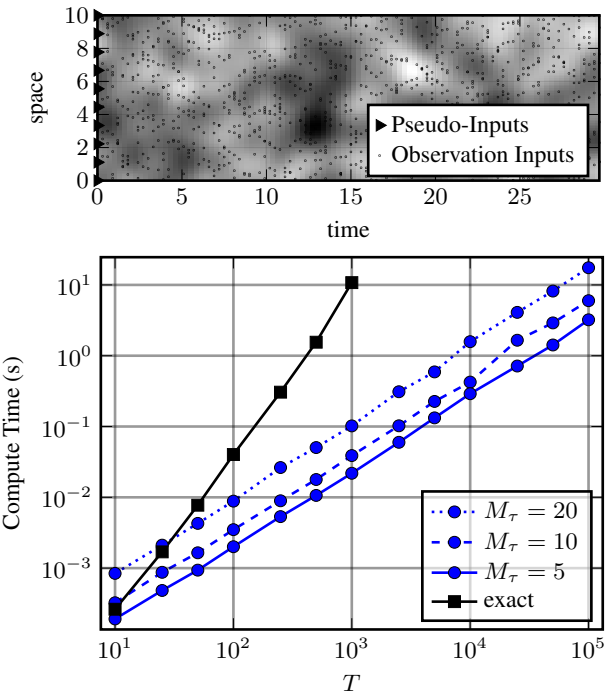

    \small
    \raggedleft
    \pgfplotsset{axis on top, scale only axis}
  
    % Set up plot.
    \setlength{\figurewidth}{.8\columnwidth}
    \setlength{\figureheight}{.333\figurewidth}    
    \includegraphics{figures/post_irregular.tikz}

    \vspace*{-5pt}

    % Timing plot.
    \setlength{\figurewidth}{.8\columnwidth}
    \setlength{\figureheight}{.75\figurewidth}            
    \includegraphics{figures/irregular_timing_plot.tikz}
    \vspace*{-20pt}
    \caption{\label{fig:irregular-plot}Arbitrary Spatial Locations. Top: Locations of (pseudo-)inputs for $\Mpertime=10$. $10$ locations in space chosen randomly at each time point. Bottom: Time to compute ELBO vs performing exact inference. ELBO tight for $\Mpertime=20$; see \cref{fig:lml-plot}.}
\end{figure}
% Line to add to auto-generated tikz figures to make the image show. Please don't delete.
% \addplot[forget plot] graphics [xmin=0.0,xmax=30.0,ymin=0.0,ymax=10.0]{post_irregular.jpg};

\begin{figure}[!t]
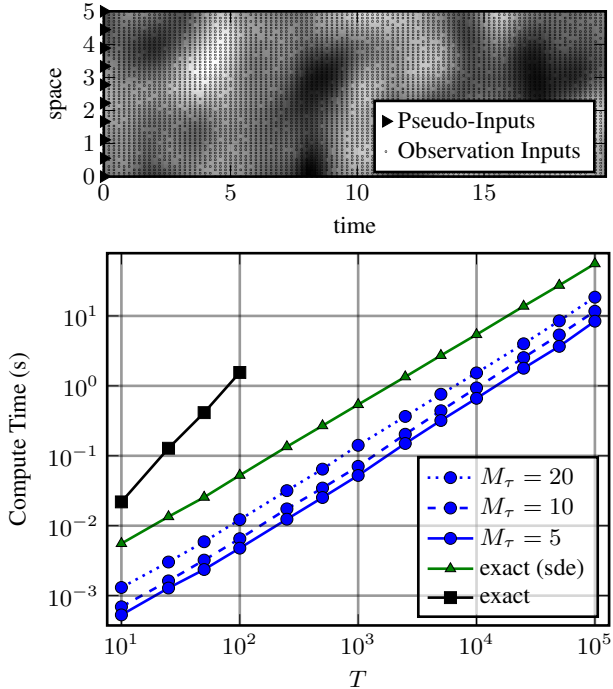

    \small
    \raggedleft
    \pgfplotsset{axis on top, scale only axis}
  
    \setlength{\figurewidth}{.8\columnwidth}
    \setlength{\figureheight}{.333\figurewidth}    
    \includegraphics{figures/post_rectlinear.tikz}

    \vspace*{-5pt}

    \setlength{\figurewidth}{.8\columnwidth}
    \setlength{\figureheight}{.75\figurewidth}
    \includegraphics{figures/regular_timing_plot.tikz}
    \vspace*{-20pt}
    \caption{\label{fig:rectilinear-plot}Grid-with-Missings. Top: Locations of \mbox{(pseudo-)inputs} -- note the grid structure with $50$ observations per time point, of which $5$ are missing. Bottom: Time to compute ELBO vs LML naively and via state space methods (\emph{sde}). ELBO tight for $\Mpertime=20$; see \cref{fig:lml-plot}.}
    \vspace*{-1em}
\end{figure}
% Line to add to auto-generated tikz figures to make the image show. Please don't delete.
% \addplot[forget plot] graphics [xmin=0.0,xmax=20.0,ymin=0.0,ymax=5.0]{post_rectlinear.jpg};

\paragraph{Relationship with Other Approximation Techniques}
There are several existing methods that could be used to scale GPs to large spatio-temporal problems beyond those already considered -- each method makes different assumptions about the kinds of problems considered, therefore making different trade-offs relative to ours.

The popular Kronecker-product methods for separable kernels explored by \citet{saatcci2012scalable} are unable to handle heteroscedastic observation noise or missing data, scale cubically in time, and require observations to lie on a rectilinear grid.
Our approach suffers none of these limitations.

\citet{wilson2015kernel} introduced a pseudo-point approximation they call \textit{Structured Kernel Interpolation} (SKI) which is closely-related to the Kronecker-product methods, but removes many of their constraints.
In particular, SKI places pseudo-points on a grid across all input dimensions, and utilises them to construct a sparse approximation to the prior covariance matrix over the data -- crucially it is local in the sense that the approximation to the covariance between the pseudo-points and any given point depends only on a handful of pseudo-points.
SKI covers the domain in a regular grid of points, which results in exponential growth in the number of pseudo-points as the number of dimensions grows.
So, while this approximation scales very well in low-dimensional settings, it does not scale to input domains comprising more than a few dimensions.
Moreover, to exploit this grid structure, separability across all dimensions is required.
\citet{gardner2018product} alleviates this exponential scaling problem, but still require that the kernel be separable across all dimensions if their approximation is to be applied.
Our approach does not suffer from this constraint as only the time dimension must be covered by pseudo-points -- there are no constraints on their spatial locations.
Naturally, that we do not perform similar approximations to SKI across the spatial dimensions means that our method will have the standard set of limitations experienced by all pseudo-point methods as the number of points in space grows.
In short, the two classes of method are applicable to different kinds of spatio-temporal problems.
They take somewhat orthogonal approaches to approximate inference, so combining them by utilising SKI across the spatial dimensions could offer the benefits of both classes of approximation in situations where SKI is applicable to the spatial component.

Similarly, approximations based on the relationship between GPs and Stochastic Partial Differential Equations \citep{whittle1963stochastic, lindgren2011explicit} could be combined with this work to improve scaling in space when the spatial kernel is in the Mat\'{e}rn family. In low-dimensional settings other standard inter-domain pseudo-point approximations such as those of \citet{hensman2017variational}, \citet{burt2020variational}, and \cite{dutordoir2020sparse} could be applied.

\section{Experiments}
\label{sec:experiments}

We view the proposed approximation to be a useful contribution if it is able to outperform the vanilla state space approximation (\cref{sec:state-space}), which is a strong baseline for the tasks we consider. To that end, we benchmark inference against synthetic data in \cref{sec:benchmarking}, on a large-scale temperature modeling task to which both the vanilla and pseudo-point state space approximations can feasibly be applied (\cref{sec:ghcn}), and finally to a problem to which it is completely infeasible to apply the vanilla state space approximation (\cref{sec:housing}).
We do not compare directly against the vanilla pseudo-point approximations of \citet{titsias2009variational} and \citet{hensman2013gaussian}.
As noted in \cref{sec:vfe}, they are asymptotically no better than exact inference for problems with long time horizons.

\subsection{Benchmarking}
\label{sec:benchmarking}

We first conduct two simple proof-of-concept experiments on synthetic data with a separable GP to verify our proposed method.
In both experiments we consider quite a large temporal extent, but only moderate spatial, since we expect the proposed method to perform well in such situations -- if the spatial extent of a data set is very large relative to the characteristic spatial variation, pseudo-point methods will struggle and, by extension, so will our method. \cref{sec:additional-experimental-details} contains additional details on the setup used, and \cref{sec:sum-separable} contains the same experiments for a sum-separable model.

\paragraph{Arbitrary Spatial Locations} \cref{fig:irregular-plot} (top) shows how inputs were arranged for this experiment; at each time $10$ spatial locations were sampled uniformly between $0$ and $10$, so $N = 10T$. The spatial location of pseudo-inputs are regular between $0$ and $10$. When using pseudo-points, we are indeed able to achieve substantial performance improvements relative to exact inference by utilising the state space methodology, while retaining a tight bound.

\paragraph{Grid-with-Missings} \cref{fig:rectilinear-plot} (top) shows how (pseudo) inputs were arranged for this experiment for $\Mpertime=10$; the same $50$ spatial locations are considered at each time point, but $5$ of the observations are dropped at random, for a total of $N_t = 45$ observations per time point -- our largest case therefore involves $N=4.5 \times 10^6$ observations.  The timing results show that we are able to compute a good approximation to the LML using roughly a third of the computation required by the standard state space approach to inference.

\subsection{Climatology Data}
\label{sec:ghcn}

\begin{figure}[!t]
    \centering
    \includegraphics[width=\columnwidth,trim=12 35 90 35,clip]{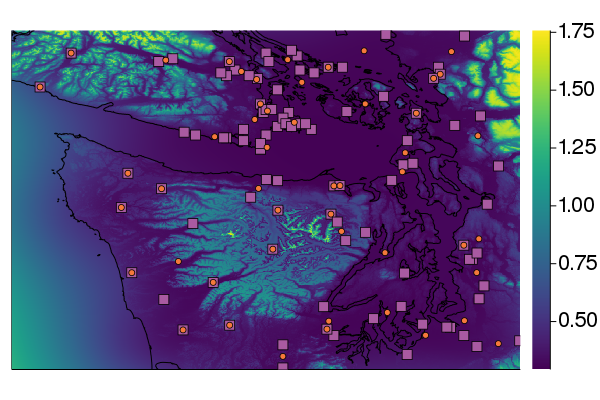}
    \caption{\label{fig:ghcn-posterior}Posterior std.\ dev.\ counterpart to  \cref{fig:posterior-mean-weather-station}. The colour scale ($0$~\protect\includegraphics[width=1cm,height=2.5mm]{./figures/viridis}~$1.75$) is relative, pink squares are weather stations, and orange dots pseudo-points.}
\end{figure}

\begin{figure*}[!t]
    \centering
    \begin{subfigure}{\columnwidth}
    \includegraphics[width=\columnwidth]{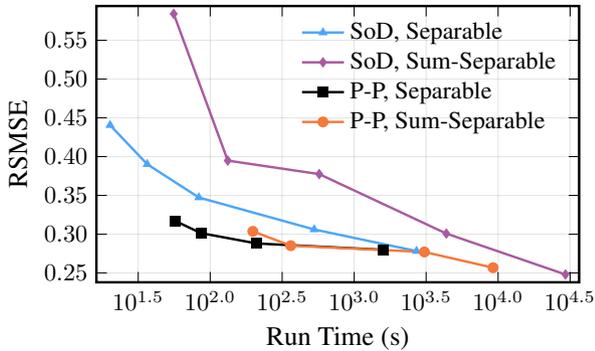}
    \end{subfigure}
    \hfill
    \begin{subfigure}{\columnwidth}
    \includegraphics[width=\columnwidth]{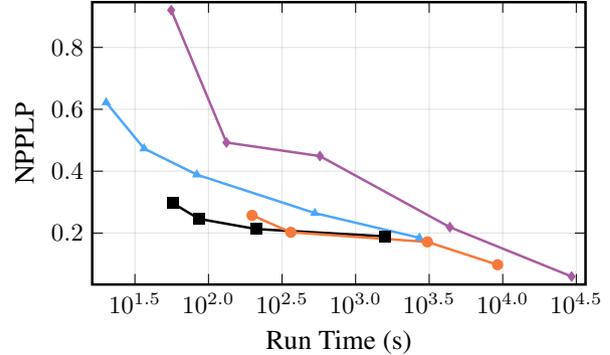}
    \end{subfigure}
    \vspace*{-10pt}
    \caption{Test Root Standardised Mean-Squared Error (RSMSE) and Negative Posterior Predictive Log Probability (NPPLP). Marked points on Pseudo-Point curves used $M \in \{ 5, 10, 20, 50 \}$ moving from left to right -- similarly for SoD markers, with the addition of $M=99$, corresponding to learning with the exact LML. Larger $M$ improves performance, but time taken to train is increased. Sum-Separable models take longer to train than Separable but can produce better results.}
    \label{fig:ghcn-speed-acc}
\end{figure*}

The Global Historical Climatology Network (GHCN) \citep{menne2012overview} comprises daily measurements of a variety of meteorological quantities, going back more than 100 years.
We combine this data with the NASA Digital Elevation Model \citep{nasadem} to model the daily maximum temperature in the region $(\ang{47}, -\ang{127})$ and $(\ang{49}, -\ang{122})$, which contains $99$ weather stations.
We utilise all data in this region since the year $2000$, training on $90\%$ ($331522$) and testing on $10\%$ ($36835$) of the data. This experiment was conducted on a workstation with a 3.60~GHz Intel i7-7820X CPU (8~cores), and 46~GB of 3000~MHz DDR3 RAM.

Two models were utilised: a simple separable model with a Mat\'{e}rn-$\frac{5}{2}$ kernel over time, and Exponentiated Quadratic over space, and a sum of two such kernels with differing length scales and variances. Additional details in \cref{sec:climatology-data-extras}.

\cref{fig:ghcn-speed-acc} compares a simple subset-of-data (SoD) approximation, which is exact when $M=99$, with the pseudo-point (P-P) approximation developed in this work.
The results demonstrate that %
{\em (i)}~the pseudo-point approximation has a more favourable speed-accuracy trade-off than the SoD, offering near exact inference in less time for a separable kernel, and
{\em (ii)}~a sum-separable model offers substantially improved results over a separable in this scenario.

\begin{figure}[!t]
    \centering
    \begin{subfigure}{.49\columnwidth}
    \begin{tikzpicture}[outer sep=0]
      \node at (0,0) {\includegraphics[width=\textwidth, trim=12 12 90 10,clip]{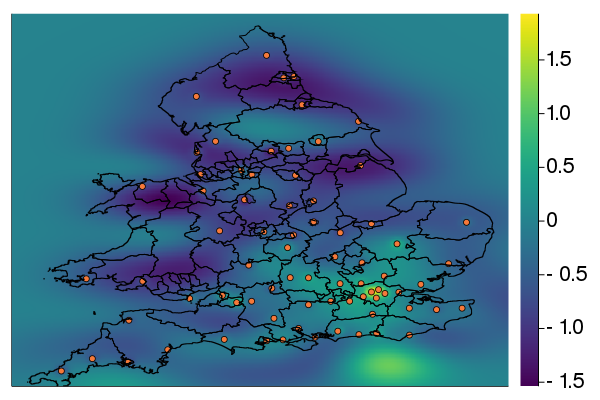}};
      \node[draw=black!70,inner sep=0.5pt] at (0,-1.8) {\includegraphics[width=2cm,height=2.5mm]{./figures/viridis}};
      \node at (-1.4,-1.8) {\scriptsize $-1.5$};
      \node at ( 1.4,-1.8) {\scriptsize $2.0$};      
    \end{tikzpicture}
    \caption{Mean}
    \end{subfigure}
    %\vspace*{-12pt}
    \hfill
    \begin{subfigure}{.49\columnwidth}    
    \begin{tikzpicture}[outer sep=0]
      \node at (0,0) {\includegraphics[width=\textwidth, trim=12 12 90 10,clip]{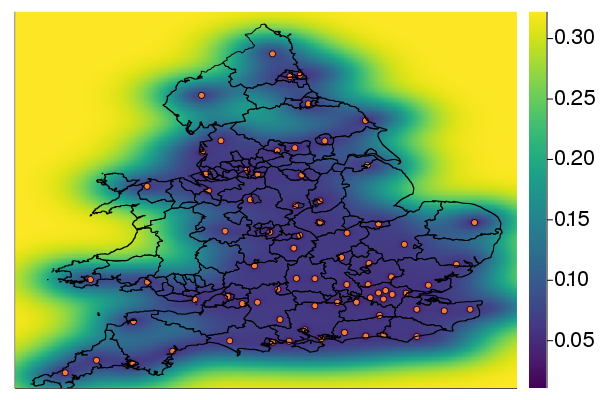}};
      \node[draw=black!70,inner sep=0.5pt] at (0,-1.8) {\includegraphics[width=2cm,height=2.5mm]{./figures/viridis}};
      \node at (-1.4,-1.8) {\scriptsize $0.0$};
      \node at ( 1.4,-1.8) {\scriptsize $0.5$};      
    \end{tikzpicture}
    \caption{Std.\ dev.}    
    \end{subfigure}
    \vspace*{-1em}
    \caption{\label{fig:housing-posterior}Apartment price posterior mean and standard deviation on a day near the end of 2020. Pseudo-point locations picked using K-means and  marked with orange dots.}
\end{figure}

\subsection{Apartment Price Data}
\label{sec:housing}

Property sales data by postcode across England and Wales are provided by \citet{housing-price-data}. There are over $10^6$ unique postcodes in England and Wales, of which a tiny proportion contain a sale on a given day. Consequently this data set has essentially arbitrary spatial locations at each point in time, which our approximation can handle, but which renders the vanilla state-space method infeasible.

We follow a similar procedure to \citet{hensman2013gaussian}, cross-referencing postcodes against a separate database \citep{camdendatabase} to obtain latitude-longitude coordinates, which we regress against the standardised logarithm of the price.
However, we train on $843766$ of the $1687536$ apartment sales between $2010$ and $2020$, and test on the remainder.
We again consider a separable and sum-separable GP that are similar to those in \cref{sec:ghcn}, but the temporal kernel is Mat\'{e}rn-$\frac{3}{2}$.
More detail in \cref{sec:apartment-data-extras}.

\cref{tab:housing-numbers} again demonstrates that a sum-separable model is able to capture more useful structure in the data than the separable model; \cref{fig:housing-posterior} shows the variability and uncertainty in the prices on an arbitrarily chosen day.

\begin{table}[!t]
    \centering
    \caption{Performance on apartment price data. $\Mpertime=75$.}
    \label{tab:housing-numbers}
    \begin{tabular}{r|cc}
\toprule
                     & RSMSE & NPPLP \\ \hline
    Separable & 0.658 &  2920 \\
Sum-Separable & 0.618 &   192 \\
\bottomrule
\end{tabular}
\end{table}

\section{Discussion}

This work shows that pseudo-point and state space approximations can be directly combined in the same model to effectively perform approximate inference and learning in sum-separable GPs, and ties up loose ends in the theory related to combining these models. This is important in spatio-temporal applications, where the model admits a form of an arbitrary-dimensional (spatial) random field with dynamics over a long temporal horizon.
Experiments on synthetic and real-world data show that this approach enables a favourable trade-off between computational complexity and accuracy.

Standard approximations for non-Gaussian observation models, such as those discussed by \citet{wilkinson2020state}, \citet{chang2020fast}, and \citet{ashman2020sparse}, can be applied straightforwardly within our approximation.
Our method represents the simplest point in a range of possible approximations. As such there are several promising paths forward to achieve further scalability beyond simply utilising hardware acceleration, including %
{\em (i)}~applying the estimator developed by \citet{hensman2013gaussian} to our method to utilise mini batches of data,
{\em (ii)}~embedding the infinite-horizon approximation introduced by \citet{solin2018infinite} to trade off some accuracy for a substantial reduction in the computational complexity of our approximation,
{\em (iii)}~removing the constraint that observations must appear at the same time as pseudo-points by utilising the method developed by \citet{adam2020doubly}.

\paragraph{Code}

\url{github.com/JuliaGaussianProcesses/TemporalGPs.jl} contains an implementation of the approximation developed in this work.

\url{github.com/willtebbutt/PseudoPointStateSpace-UAI-2021} contains code built on top of TemporalGPs.jl to reproduce the experiments.

\begin{contributions}
    WT conceived the idea, implemented models, and ran the experiments. All authors helped develop the idea, write the paper, and devise experiments.
\end{contributions}

\begin{acknowledgements}
We thank Adri\a`{a} Garriga-Alonso, Wessel Bruinsma, Matt Ashman, and anonymous reviewers for invaluable feedback.
Will Tebbutt is supported by Deepmind and Invenia Labs.
Arno Solin acknowledges funding from the Academy of Finland (grant id 324345). Richard E.~Turner is supported by Google, Amazon, ARM, Improbable, Microsoft Research and EPSRC grants EP/M0269571 and EP/L000776/1.

\end{acknowledgements}

\newpage
\bibliography{bio.bib}

\onecolumn

\def\toptitlebar{\hrule height1pt \vskip .25in} 
\def\bottomtitlebar{\vskip .22in \hrule height1pt \vskip .3in} 

\newcommand{\nipstitle}[1]{%
    \phantomsection%\hsize\textwidth\linewidth\hsize%
    \vskip 0.1in%
    \toptitlebar%
    \begin{minipage}{\textwidth}%
        \centering{\Large\bf #1\par}%
    \end{minipage}%
    \bottomtitlebar%
    %\addcontentsline{toc}{section}{#1}%
}

\clearpage
\normalsize

\nipstitle{
    {Supplementary Material:} \\
    Combining Pseudo-Point and State Space Approximations \\
    for Sum-Separable Gaussian Processes
}

\pagestyle{empty}

\appendix

\newcommand{\inputdomainaux}{\bar{\inputdomain}}

\newcommand{\xaux}{\bar{x}}

\newcommand{\kerneltd}{\kernel_{\timevar \dimvar}}
\newcommand{\kernelspace}{\kernel_{\spacevar}}

\newcommand{\transitionoperator}{\mathcal{A}}

\section{Conditional Independence Results}
\label{sec:cond-indep-results}

\subsection{The Conditional Independence Structure of Collections of Points}
\label{sec:cond-indep-collections}

The following lemma establishes an analogue of the conditional independence result introduced by \citet{o1998markov}, which applies to individual points, to collections of points in a separable Gaussian process.

\begin{lemma}\label{lma:cond-indep-extension}
    Let $\domainX$ and $\domainY$ be sets, $f \sim \dGP{0, \kernel}$ where $\Func{\kernel}{(x, y), (x^\prime, y^\prime)} := \Func{\kernelx}{x, x^\prime} \Func{\kernely}{y, y^\prime}$, $x, x^\prime \in \domainX$ and $y, y^\prime \in \domainY$, and $\kernelx$ and $\kernely$ are non-degenerate, meaning covariance matrices constructed using them are invertible.
    Then for finite sets $\domainX_1, \domainX_2 \subset \domainX$, $\domainY_1, \domainY_2 \subset \domainY$, and sets of random variables
    \begin{align}
        &\setA = \{f(x, y) \mid x \in \domainX_1, y \in \domainY_1\}, \nonumber \\
        &\setB = \{f(x, y) \mid x \in \domainX_2, y \in \domainY_2\}, \nonumber \\
        &\setC = \{f(x, y) \mid x \in \domainX_2, y \in \domainY_1\}, \nonumber
    \end{align}
    it is the case that
    \begin{equation}
        \setA \bigCI \setB \mid \setC. \\
    \end{equation}
\end{lemma}

\begin{proof}
    Since $\setA$, $\setB$, and $\setC$ are jointly Gaussian, it is sufficient to show that the conditional covariance $\cov{\setA, \setB \mid \setC}$ is always $\mathbf{0}$.
    Assign an arbitrary ordering to the elements in each of $\domainX_1, \domainX_2, \domainY_1$, and $\domainY_2$, and let $\covmat_{\domainX_i \domainX_j}$ be the covariance matrix obtained by evaluating $\kernelx$ at each pair of points $\domainX_i$ and $\domainX_j$, such that the $(p,q)^{th}$ element of $\covmat_{\domainX_i \domainX_j}$ is $\kernelx$ evaluated at the $p^{th}$ and $q^{th}$ elements of $\domainX_i$ and $\domainX_j$ respectively.
    Let $\covmat_{\domainY_i, \domainY_j}$ be analogously defined for $\kernely$ and $\domainY_i$, $\domainY_j$.
    Denote the Kronecker product by $\kron$, and order the elements of $\setA$, $\setB$, and $\setC$ such that
    \begin{align}
        \cov{\setA, \setB} &= \covmat_{\domainX_1 \domainX_2} \kron \covmat_{\domainY_1 \domainY_2}, \nonumber \\
        \cov{\setA, \setC} &= \covmat_{\domainX_1 \domainX_2} \kron \covmat_{\domainY_1 \domainY_1}, \nonumber \\
        \cov{\setC, \setB} &= \covmat_{\domainX_2 \domainX_2} \kron \covmat_{\domainY_1 \domainY_2}, \nonumber \\
        \cov{\setC} &= \covmat_{\domainX_2 \domainX_2} \kron \covmat_{\domainY_1 \domainY_1}. \nonumber
    \end{align}
    $\kernelx$ and $\kernely$ are non-degenerate, so $\covmat_{\domainX_2 \domainX_2}$, $\covmat_{\domainY_1 \domainY_1}$, and $\cov{\setC}$ are invertible, and the conditional covariance is
    \begin{align}
        &\cov{\setA, \setB \mid \setC} \nonumber \\
        &\quad= \cov{\setA, \setB} - \cov{\setA, \setC} \inv{\cov{\setC}} \cov{\setC, \setB} \nonumber \\
        &\quad= \covmat_{\domainX_1 \domainX_2} \kron \covmat_{\domainY_1 \domainY_2} - \left( \covmat_{\domainX_1 \domainX_2} \kron \covmat_{\domainY_1 \domainY_1} \right) \invbr{\covmat_{\domainX_2 \domainX_2} \kron \covmat_{\domainY_1 \domainY_1}} \left( \covmat_{\domainX_2 \domainX_2} \kron \covmat_{\domainY_1 \domainY_2} \right) \nonumber \\
        &\quad= \covmat_{\domainX_1 \domainX_2} \kron \covmat_{\domainY_1 \domainY_2} - \left( \covmat_{\domainX_1 \domainX_2} \inv{\covmat_{\domainX_2 \domainX_2}} \covmat_{\domainX_2 \domainX_2}  \right) \left( \covmat_{\domainY_1 \domainY_1} \inv{\covmat_{\domainY_1 \domainY_1}} \covmat_{\domainY_1 \domainY_2} \right) \nonumber \\
        &\quad= \covmat_{\domainX_1 \domainX_2} \kron \covmat_{\domainY_1 \domainY_2} - \covmat_{\domainX_1 \domainX_2} \kron \covmat_{\domainY_1 \domainY_2} \nonumber \\
        &\quad= \mathbf{0}. \nonumber \qedhere
    \end{align}
\end{proof}

This result is depicted in \cref{fig:cond-indep-extension} -- specifically letting $\setA$ be the red squares, $\setB$ the blue squares, and $\setC$ the black dots. $\domainX_1$ are the $x$-coordinates of the red squares, $\domainX_2$ the $x$-coordinates of the blue squares / black circles, $\domainY_1$ the $y$-coordinates of the red squares / black circles, and $\domainY_2$ the $y$-coordinates of the blue squares.

\cref{lma:cond-indep-extension} establishes that a GP being separable implies the presented conditional independence between collections of points.
\citet{o1998markov} goes further in the single-point case, also showing that the conditional independence property implies separability, hence showing that the separability and conditional independence statements are equivalent.
While it seems plausible that such an equivalence could be established for collections of points, such a result is not needed in this work, and is therefore not pursued.

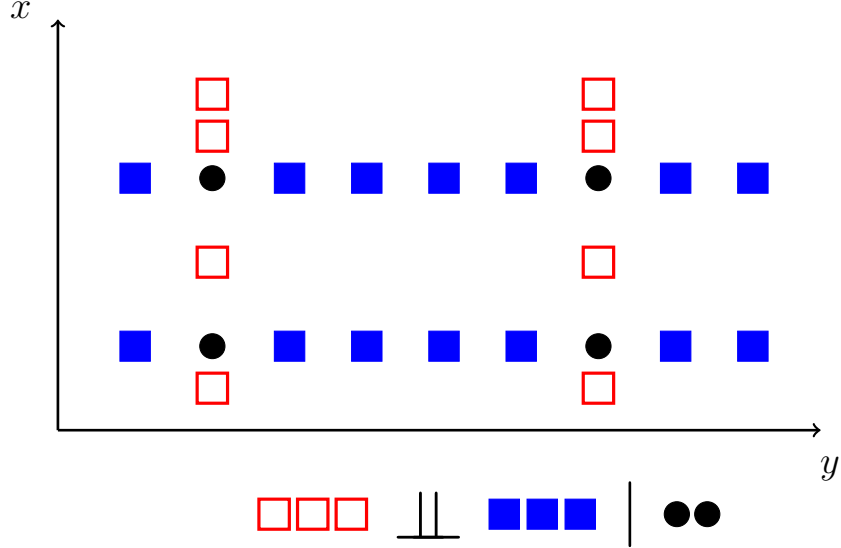
\begin{figure}[h]
    \centering
    \begin{tikzpicture}

        \def\myheight{2.2in}
        \def\mywidth{4in}

        \centering

        % Specify corners.
        \node (lhstopcorner) at (0, \myheight) {};
        \node (rhsbottomcorner) at (\mywidth, 0) {};

        % Draw axes.
        \draw[->, line width=1] (0, 0) -- (lhstopcorner);
        \draw[->, line width=1] (0, 0) -- (rhsbottomcorner);

        % Draw blue circles.
        \foreach \r in {1, 3}
            \foreach \t in {1, 3, 4, 5, 6, 8, 9} {
                \node at (\mywidth * \t / 10, \myheight * \r / 5) [bluecircle] {};
            }

        % Draw black dots.
        \foreach \t in {2, 7}
            \foreach \r in {1, 3} {
                \node at (\mywidth * \t / 10, \myheight * \r / 5) [blackcircle] {};
            }
        % Draw red squares.
        \foreach \r in {1, 4, 7, 8}
            \foreach \t in {2, 7} {
                \node at (\mywidth * \t / 10, \myheight * \r / 10) [redcircle] {};
            }

        % Conditional independence statement.
        \node (foo) at (\mywidth * 1.4 / 5, -\myheight * 1 / 5) [redcircle] {};
        \node (foo1) at (0.2in + \mywidth * 1.4 / 5, -\myheight * 1 / 5) [redcircle] {};
        \node (foo2) at (0.4in + \mywidth * 1.4 / 5, -\myheight * 1 / 5) [redcircle] {};
        \node (bar) [right=0.2 of foo2] {\Huge $\bigCI$};
        \node (baz) [right=0.2 of bar] [bluecircle] {};
        \node (baz1) [right=0.7 of bar] [bluecircle] {};
        \node (baz2) [right=1.2 of bar] [bluecircle] {};
        \node (bleh) [right=0.2 of baz2] {\Huge $\mid$};
        \node (hmm) [right=0.2 of bleh] [blackcircle] {};
        \node (hmm1) [right=0.6 of bleh] [blackcircle] {};

        \node (leftmost) [left=0.1 of foo] {};
        \node (rightmost) [right=0.1 of hmm] {};

        \draw (\mywidth * 1 / 5, \myheight * 4 / 5) node {};

        \node (timevar) [below=0.1 of rhsbottomcorner] {\Large $y$};
        \node (spacevar) [left=0.1 of lhstopcorner] {\Large $x$};

    \end{tikzpicture}
    \caption{\label{fig:cond-indep-extension}Under a separable GP prior, the random variables at the red squares are conditionally independent of those at the blue squares given those at the black circles.}
\end{figure}

\subsection{Separability of the State-Space Approximation}

\citet{solin2016stochastic} shows in chapter 5 that a separable spatio-temporal GP $f$, whose time-kernel has Markov structure, can be expressed as another GP $\faux$ with some auxiliary dimensions.
The kernel of $\faux$ is not given explicitly -- instead it is expressed in terms of an infinite-dimensional Kalman filter.
Consequently, it is unclear without further investigation whether or not the kernel over $\faux$ is separable.
We show that, grouping together time and the latent dimension, it in fact separates over space and the grouped dimensions.

Let $\timevar \in \reals$ denote a point in time, and samples from $\faux(\spacevar, \timevar, \dimvar)$ be the random variable in $f$ associated with the spatial location $\spacevar$, time point $\timevar$, and latent dimension $\dimvar$.
Denote $\faux_\timevar = \Func{\faux}{\cdot, \timevar, \cdot}$, then \citet{solin2016stochastic} shows that
\begin{align}
    \faux_\timevar \sim \dGP{0, \Func{\kernelspace}{\spacevar, \spacevar^\prime} \Func{\alpha_{\timevar}}{d, d^\prime}} \label{eqn:marginal-kernel}
\end{align}
and, for any $\timevar^\prime$, kernel $\alpha_{\timevar} : \dimdomain \times \dimdomain \to \reals$ (isomorphic to a $D \times D$ matrix), kernel over space $\kernel_\spacevar$.
Let $\transitionoperator_{\timevar \to \timevar^\prime} : \inputdomain \times \dimdomain \to \inputdomain \times \dimdomain$ be the linear transition operator,
\begin{align}
    \faux_{\timevar^\prime} &= \transitionoperator_{\timevar \to \timevar^\prime} \faux_{\timevar} + q_{\timevar \to \timevar^\prime} \label{eqn:conditional-state-space} \\
    \Func{(\transitionoperator_{\timevar \to \timevar^\prime} \faux_{\timevar})}{\spacevar, \dimvar} &= \sum_{j=1}^D [\transition_{\timevar \to \timevar^\prime}]_{dj} \Func{\faux_\timevar}{\spacevar, j}, \label{eqn:lin-op-structure}
\end{align}
where each $q_{\timevar \to \timevar^\prime}$ is an independent GP, samples from which are functions $\spacedomain \times \dimdomain \to \reals$, and $\transition_{\timevar \to \timevar^\prime}$ is a $D \times D$ matrix of real numbers.

\begin{lemma}\label{lma:sep-state-space}
Let $\faux \sim \dGP{0, \kernelaux}$, the distribution over any time-marginal $\faux_\timevar$ be defined according to \cref{eqn:marginal-kernel}, and the conditional distribution over $\faux_{\timevar^\prime}$ given $\faux_{\timevar}$ \cref{eqn:conditional-state-space}. It follows that $\kernelaux$ is separable, and of the form
\begin{equation}
    \Func{\kernelaux}{(\spacevar, \timevar, \dimvar), (\spacevar^\prime, \timevar^\prime, \dimvar^\prime)} = \Func{\kernelspace}{\spacevar, \spacevar^\prime} \Func{\kerneltd}{(\timevar, \dimvar), (\timevar^\prime, \dimvar^\prime)}
\end{equation}
for some kernel $\kerneltd$.
\end{lemma}

\begin{proof}
    \begin{align}
        \Func{\kernelaux}{(\spacevar, \timevar, \dimvar), (\spacevar^\prime, \timevar^\prime, \dimvar^\prime)} :=&\, \cov{\Func{\faux}{\spacevar, \timevar, \dimvar}, \Func{\faux}{\spacevar^\prime, \timevar^\prime, \dimvar^\prime}} \nonumber \\
        =&\, \cov{\Func{\faux_\timevar}{\spacevar, \dimvar}, \Func{\faux_{\timevar^\prime}}{\spacevar^\prime, \dimvar^\prime}} \nonumber \\
        =&\, \Expect{}{\Func{\faux_\timevar}{\spacevar, \dimvar} \Func{\faux_{\timevar^\prime}}{\spacevar^\prime, \dimvar^\prime}} \nonumber \\
        =&\, \Expect{}{\Func{\faux_\timevar}{\spacevar, \dimvar} \{ \Func{(\transitionoperator_{\timevar \to \timevar^\prime} \faux_{\timevar})}{\spacevar^\prime, \dimvar^\prime} + \Func{q_{\timevar \to \timevar^\prime}}{\spacevar^\prime, \dimvar^\prime} \}} \nonumber \\
        =&\, \Expect{}{\Func{\faux_\timevar}{\spacevar, \dimvar} \{ \Func{(\transitionoperator_{\timevar \to \timevar^\prime} \faux_{\timevar})}{\spacevar^\prime, \dimvar^\prime}}. \nonumber
    \end{align}
    where the penultimate equality follows from independence.
    Applying \cref{eqn:lin-op-structure} yields
    \begin{align}
        \Func{\kernelaux}{(\spacevar, \timevar, \dimvar), (\spacevar^\prime, \timevar^\prime, \dimvar^\prime)} =&\, \sum_{j=1}^D [\transition_{\timevar \to \timevar^\prime}]_{d^\prime j} \Expect{}{\Func{\faux_{\timevar}}{\spacevar, \dimvar} \Func{\faux_{\timevar}}{\spacevar^\prime, j} } \nonumber \\
        =&\, \sum_{j=1}^D [\transition_{\timevar \to \timevar^\prime}]_{d^\prime j} \Expect{}{\Func{\faux}{\spacevar, \timevar, \dimvar} \Func{\faux}{\spacevar^\prime, \timevar, j} } \nonumber
    \end{align}
    Applying \cref{eqn:marginal-kernel} yields
    \begin{align}
        \Func{\kernelaux}{(\spacevar, \timevar, \dimvar), (\spacevar^\prime, \timevar^\prime, \dimvar^\prime)} =&\, \sum_{j=1}^D [\transition_{\timevar \to \timevar^\prime}]_{d^\prime j} \Func{\kernelspace}{\spacevar, \spacevar^\prime} \Func{\alpha_{\timevar}}{(\dimvar, j} \nonumber \\
        =&\, \Func{\kernelspace}{\spacevar, \spacevar^\prime} \sum_{j=1}^D [\transition_{\timevar \to \timevar^\prime}]_{d^\prime j}  \Func{\alpha_{\timevar}}{\dimvar, j} \nonumber \\
        =&\, \Func{\kernelspace}{\spacevar, \spacevar^\prime} \Func{\kerneltd}{(\timevar, \dimvar), (\timevar^\prime, \dimvar^\prime)} \nonumber
    \end{align}
    where
    \begin{equation}
        \Func{\kerneltd}{(\timevar, \dimvar), (\timevar^\prime, \dimvar^\prime)} := \sum_{j=1}^D [\transition_{\timevar \to \timevar^\prime}]_{d^\prime j}  \Func{\alpha_{\timevar}}{\dimvar, j}. \nonumber
    \end{equation}
\end{proof}

\subsection{Conditional Independence Structure of Observations and Pseudo-Points Under a Separable Prior}
\label{sec:cond-indep-pseudo}

\newcommand{\zobsaux}{\bar{\vec{z}}}
\newcommand{\spacepseudoinputs}{\mathcal{Z}_\spacevar}
\newcommand{\dimpseudoinputs}{\{1, ..., D\}}

\newcommand{\timeinputs}{\mathcal{T}}

Assume that the set of pseudo-inputs $\zobsaux$ form the rectilinear grid
\begin{align}
    \zobsaux := \spacepseudoinputs \times \timeinputs \times \dimpseudoinputs
\end{align}
where $\times$ denotes the Cartesian product, and $\spacepseudoinputs \subset \inputdomain$ and $\timeinputs \subset \reals$ are the finite sets of spatial and temporal locations at which pseudo-inputs are present, with sizes $M_\timevar$ and $T$ respectively.
Let the pseudo-points be
\begin{equation}
    \uobsaux := \{ \faux(\spacevar, \timevar, \dimvar) 
    \mid (\spacevar, \timevar, \dimvar) \in \zobsaux \}.
\end{equation}
Furthermore, let
\begin{equation}
    \uobs_\timevar := \{\faux(\spacevar, \timevar, 1) \mid \spacevar \in \spacepseudoinputs \},
\end{equation}
then $\uobsaux \setminus \uobs_\timevar$ is the collection of all pseudo-points not in $\uobs_\timevar$.

Let $\inputdomain_1, ..., \inputdomain_{T} \subset \inputdomain$ be finite sets of of points in space, one for each point in $\timeinputs$.
The set of points
\begin{equation}
    \xinp_\timevar := \{ (\spacevar, \timevar, 1) \mid \spacevar \in \inputdomain_\timevar \}
\end{equation}
are the elements of $\faux$ which are observed (noisily) at time $\timevar$, so let
\begin{equation}
    \fobs_\timevar := \{ \faux(\spacevar, \timevar, \dimvar) \mid (\spacevar, \timevar, \dimvar) \in \xinp_\timevar \}
\end{equation}

It is now possible to the present the key result:
\begin{theorem}\label{thm:cond-indep}
    \begin{equation}
        \fobs_\timevar \bigCI \uobsaux \setminus \uobs_\timevar \mid \uobs_\timevar.
    \end{equation}
    That is: $\fobs_\timevar$ is conditionally independent of all pseudo-points not in the first latent dimension of $\faux$ at time $\timevar$, given all of the pseudo-points in the first latent dimension of $\faux$ at time $\timevar$ -- \cref{fig:cond-indep-pseudo-points} visualises this property.
\end{theorem}

\begin{proof}
    Let $\domainX := \inputdomain$ and $\domainY := \timedomain \times \dimdomain$ -- i.e. group together time and latent dimension -- and (abusing notation) let
    \begin{equation}
        \faux(\spacevar, (\timevar, \dimvar)) := \faux(\spacevar, \timevar, \dimvar), \quad \spacevar \in \domainX, \quad (\timevar, \dimvar) \in \domainY.
    \end{equation}
    By \cref{lma:sep-state-space} the kernel over $\faux$ is separable across $\domainX$ and $\domainY$.
    Applying \cref{lma:cond-indep-extension} to $\faux$ and
    \begin{align}
        \domainX_1 := \inputdomain_\timevar, \quad \domainX_2 := \spacepseudoinputs, \quad \domainY_1 := \{(\timevar, 1)\}, \quad \textrm{ and } \domainY_2 := [ \timeinputs \times \{1, ..., D\} ] \setminus \domainY_1.
    \end{align}
    yields the desired result, as $\faux_{\domainX_1, \domainY_1} = \fobs_\timevar$, $\faux_{\domainX_2, \domainY_2} = \uobsaux \setminus \uobs_\timevar$, and $\faux_{\domainX_2, \domainY_1} = \uobs_\timevar$.
\end{proof}

This result is depicted in \cref{fig:cond-indep-pseudo-points}. 
Of the entire 3 dimensional grid of pseudo-points, only those in the first latent dimension at the same time as $\fobs_\timevar$ are needed to achieve conditional independence from all others.

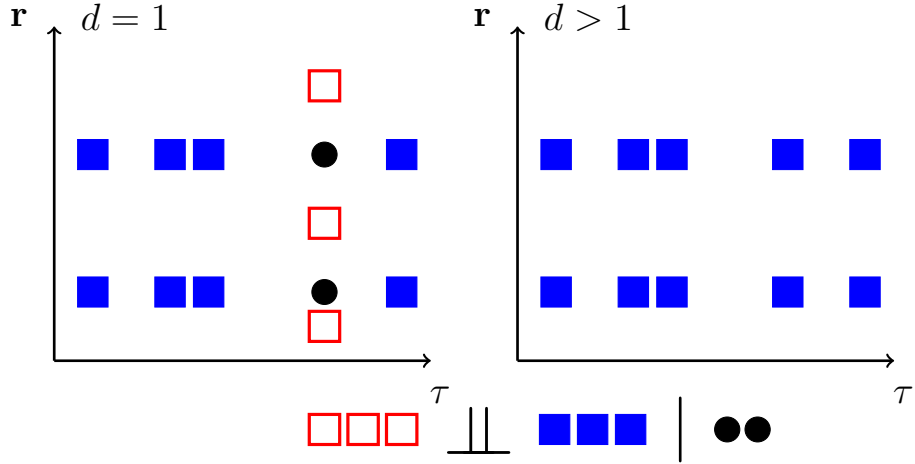
\begin{figure}[h]
    \centering
    \begin{tikzpicture}

        \def\myheight{1.8in}
        \def\mywidth{2in}

        \centering

        % Specify corners.
        \node (lhstopcorner) at (0, \myheight) {};
        \node (rhsbottomcorner) at (\mywidth, 0) {};

        % Draw axes.
        \draw[->, line width=1] (0, 0) -- (lhstopcorner);
        \draw[->, line width=1] (0, 0) -- (rhsbottomcorner);

        % Draw blue circles.
        \foreach \r in {1, 3}
            \foreach \t in {1, 3, 4, 9} {
                \node at (\mywidth * \t / 10, \myheight * \r / 5) [bluecircle] {};
            }

        % Draw black dots.
        \foreach \r in {1, 3}
            \foreach \t in {7}{
                \node at (\mywidth * \t / 10, \myheight * \r / 5) [blackcircle] {};
            }
        % Draw red circles.
        \foreach \r in {1, 4, 8}
            \foreach \t in {7} {
                \node at (\mywidth * \t / 10, \myheight * \r / 10) [redcircle] {};
            }

        % Conditional independence statement.
        \node (foo) at (\mywidth * 3.5 / 5, -\myheight * 1 / 5) [redcircle] {};
        \node (foo1) at (0.2in + \mywidth * 3.5 / 5, -\myheight * 1 / 5) [redcircle] {};
        \node (foo2) at (0.4in + \mywidth * 3.5 / 5, -\myheight * 1 / 5) [redcircle] {};
        \node (bar) [right=0.2 of foo2] {\Huge $\bigCI$};
        \node (baz) [right=0.2 of bar] [bluecircle] {};
        \node (baz1) [right=0.7 of bar] [bluecircle] {};
        \node (baz2) [right=1.2 of bar] [bluecircle] {};
        \node (bleh) [right=0.2 of baz2] {\Huge $\mid$};
        \node (hmm) [right=0.2 of bleh] [blackcircle] {};
        \node (hmm1) [right=0.6 of bleh] [blackcircle] {};

        \node (leftmost) [left=0.1 of foo] {};
        \node (rightmost) [right=0.1 of hmm] {};

        \draw (\mywidth * 1 / 5, \myheight * 4 / 5) node {};

        \node (timevar) [below=0.1 of rhsbottomcorner] {\Large $\timevar$};
        \node (spacevar) [left=0.1 of lhstopcorner] {\Large $\spacevar$};
        \node (dimvar) [right=0.1 of lhstopcorner] {\Large $\dimvar = 1$};

        % Begin second bit of figure.
        \def\mywidthnext{2.4in}

        % Specify corners.
        \node (bottomleftcorner) at (\mywidthnext, 0) {};
        \node (lhstopcornerb) at (\mywidthnext, \myheight) {};
        \node (rhsbottomcornerb) at (\mywidthnext + \mywidth, 0) {};

        % Draw axes.
        \draw[->, line width=1] (\mywidthnext, 0) -- (lhstopcornerb);
        \draw[->, line width=1] (\mywidthnext, 0) -- (rhsbottomcornerb);

        % Draw blue circles.
        \foreach \r in {1, 3}
            \foreach \t in {1, 3, 4, 7, 9} {
                \node at (\mywidthnext + \mywidth * \t / 10, \myheight * \r / 5) [bluecircle] {};
            }

        \node (timevar) [below=0.1 of rhsbottomcornerb] {\Large $\timevar$};
        \node (spacevar) [left=0.1 of lhstopcornerb] {\Large $\spacevar$};
        \node (dimvar) [right=0.1 of lhstopcornerb] {\Large $\dimvar > 1$};

    \end{tikzpicture}
    \caption{\label{fig:cond-indep-pseudo-points}Slices of the 3-dimensional rectilinear grid of pseudo-points / inputs, as well as inputs of observations, depicting of the conditional independence structure presented in \cref{thm:cond-indep}. Unfilled red squares correspond to $\fobs_\timevar$, black circles to $\uobs_\timevar$, and filled blue squares to $\uobsaux \setminus \uobs_\timevar$. The left-hand side corresponds to $\dimvar=1$, while the right-hand side corresponds to $\dimvar > 1$. Notice that $\uobs_\timevar$ and $\fobs_\timevar$ only appear in the $d=1$ slice.}
\end{figure}

\subsection{Conditional Independence Structure under a Sum-Separable Prior}
\label{sec:cond-indep-sum-separable}

\newcommand{\fauxp}{\faux^p}
\newcommand{\fauxs}{\faux^s}

Recall that a sum-separable GP $\fauxs$ is defined to be a GP of the form
\begin{equation}
    \fauxs := \sum_{p=1}^P \fauxp, \quad \fauxp \sim \dGP{0, \kernelaux^p},
\end{equation}
where each $\fauxp$ is an independent separable GP.
Locate rectilinear grids of pseudo-inputs in each of the separable processes:
\begin{align}
    \zobsaux^p := \spacepseudoinputs^p \times \timeinputs \times \dimdomain
\end{align}
where $\spacepseudoinputs^p$ are a collection of points in space which are specific to each $p$.
Each of these $P$ grids of points is of the same form as those utilised for separable processes previously.
Define sets of pseudo-points for each process:
\begin{align}
    \uobsaux^p :=&\, \{ \fauxp(\spacevar, \timevar, \dimvar) \mid (\spacevar, \timevar, \dimvar) \in \zobsaux^p \}, \\
    \uobs_\timevar^p :=&\, \{ \fauxp(\spacevar, \timevar, 1) \mid \spacevar \in \spacepseudoinputs^p \}, \\
    p \in& \, \{1, ..., P \},
\end{align}
and sets containing all of the of pseudo-points through the union of the above process-specific pseudo-points:
\begin{align}
    \uobsaux :=&\, \cup_{p=1}^P \uobsaux^p, \\
    \uobs_\timevar :=&\, \cup_{p=1}^P \uobs_\timevar^p.
\end{align}

Furthermore, let
\begin{align}
    \fobs_\timevar^p :=&\, \{ \fauxp(\spacevar, \timevar, 1) \mid \spacevar \in \inputdomain_\timevar \} \\
    \fobs_\timevar^s :=&\, \{ \fauxs(\spacevar, \timevar, 1) \mid \spacevar \in \inputdomain_\timevar \}
\end{align}

\begin{theorem}[Conditional Independence in Sum-Separable GPs]
    \begin{equation}
        \fobs_\timevar^s \bigCI \uobsaux \setminus \uobs_\timevar \mid \uobs_\timevar \nonumber
    \end{equation}
\end{theorem}
\begin{proof}
    As in \cref{thm:cond-indep}, it suffices to show that $\cov{\fobs_\timevar^s, \uobsaux \setminus \uobs_\timevar \mid \uobs_\timevar} = \mathbf{0}$.
    It is the case that
    \begin{equation}
        \cov{\faux_s(\spacevar, \timevar, \dimvar), \faux_p(\spacevar^\prime, \timevar^\prime, \dimvar^\prime)} = \cov{\faux_p(\spacevar, \timevar, \dimvar), \faux_p(\spacevar^\prime, \timevar^\prime, \dimvar^\prime)},
    \end{equation}
    and the covariance between any points in $\faux_p$ and $\faux_{p^\prime}$ is $0$ if $p \neq p^\prime$,
    so
    \begin{align}
        \cov{\fobs_\timevar^s, \uobsaux \setminus \uobs_\timevar} &= \begin{bmatrix}
            \cov{\fobs_\timevar^1, \uobsaux^1 \setminus \uobs_\timevar^1} & \hdots & \cov{\fobs_\timevar^P, \uobsaux^P \setminus \uobs_\timevar^P}
        \end{bmatrix} \nonumber \\
        \cov{\fobs_\timevar^s, \uobs_\timevar} &= \begin{bmatrix}
            \cov{\fobs_\timevar^1, \uobs_\timevar^1} & \hdots & \cov{\fobs_\timevar^P, \uobs_\timevar^P}
        \end{bmatrix} \nonumber \\
        \cov{\uobs_\timevar} &= \begin{bmatrix}
            \cov{\uobs_\timevar^1} & & \mathbf{0} \\
             & \ddots & \\
            \mathbf{0} & & \cov{\uobs_\timevar^P}
        \end{bmatrix} \nonumber \\
        \cov{\uobs_\timevar, \uobsaux \setminus \uobs_\timevar} &= \begin{bmatrix}
            \cov{\uobs_\timevar^1, \uobsaux \setminus \uobs_\timevar^1} & & \mathbf{0} \\
             & \ddots & \\
            \mathbf{0} & & \cov{\uobs_\timevar^P, \uobsaux \setminus \uobs_\timevar^P}
        \end{bmatrix}. \nonumber
    \end{align}
    Therefore
    \begin{align}
        \cov{\fobs_\timevar^s, \uobsaux \setminus \uobs_\timevar \mid \uobs_\timevar} &= \cov{\fobs_\timevar^s, \uobsaux \setminus \uobs_\timevar} - \cov{\fobs_\timevar^s, \uobs_\timevar} [\cov{\uobs_\timevar}]^{-1} \cov{\uobs_\timevar, \uobsaux \setminus \uobs_\timevar} \nonumber \\
        &= \begin{bmatrix}
            \cov{\fobs_\timevar^1, \uobsaux^1 \setminus \uobs_\timevar^1 \mid \uobs_\timevar^1} & \hdots & \cov{\fobs_\timevar^P, \uobsaux^P \setminus \uobs_\timevar^P \mid \uobs_\timevar^P}
        \end{bmatrix} \nonumber \\
        &= \begin{bmatrix}
            \mathbf{0} & \hdots & \mathbf{0}
        \end{bmatrix}
    \end{align}
    where the final equality follows from \cref{thm:cond-indep}.
\end{proof}

\subsection{Block-Diagonal Structure}
\label{sec:block-diag-structure}

Furthermore, this conditional independence property implies that
\begin{equation}
    \covmat_{\fobs_t \uobsaux } \inv{\covmat}_{\uobsaux} = \begin{bmatrix}
        \mathbf{0} & \hdots & \covmat_{\fobsaux_t \uobsaux_t} \inv{\covmat}_{\uobsaux_t} & \hdots & \mathbf{0}
    \end{bmatrix}. \label{eqn:mean-prop}
\end{equation}
This is easily proven by considering that, were it not the case, then
\begin{equation}
    \Expect{}{\fobs_t \mid \uobs} \neq \Expect{}{\fobs_t \mid \uobs_t},
\end{equation}
for any non-zero $\uobs$, which is a contradiction.
It follows from repeated application of \cref{eqn:mean-prop} that the larger matrix $\covmat_{\fobs \uobsaux} \inv{\covmat}_{\uobsaux}$ is block-diagonal, and is given by
\begin{equation}
    \covmat_{\fobs \uobsaux} \inv{\covmat}_{\uobsaux} := \begin{bmatrix}
        \covmat_{\fobs_1 \uobsaux } \inv{\covmat}_{\uobsaux} \\
        \vdots \\
        \covmat_{\fobs_T \uobsaux } \inv{\covmat}_{\uobsaux}
    \end{bmatrix} =
    \begin{bmatrix}
        \covmat_{\fobs_1 \uobsaux_1} \inv{\covmat}_{\uobsaux_1} & & \mathbf{0} \\
        & \ddots & \\
        \mathbf{0} & & \covmat_{\fobs_T \uobsaux_T} \inv{\covmat}_{\uobsaux_T}.
    \end{bmatrix}
\end{equation}

Similar manipulations reveal that the same property holds in the sum-separable case.

\section{Conditional Independence Properties of Optimal Approximate Observation Models}
\label{sec:optimal-approximate-observation-models}

Conditional independence structure in the observation model is reflected in the optimal approximate posterior, regardless the precise form of the observation model (Gaussian, Bernoulli, etc).
Moreover, for Gaussian observation models, it is possible to find the model in which performing exact inference yields the optimal approximate posterior.
These two properties are derived in the following two subsections.

\subsection{Conditional Independence Structure}
\label{sec:optimal-approx-posterior-structure}

This is only a slight extension of the result of \citet{seeger1999bayesian} and \citet{opper2009variational}, which generalises from reconstruction terms which depend on only a one dimensional marginal of the Gaussian in question, to non-overlapping multi-dimensional marginals of arbitrary size.
Let
\begin{equation}
    \Prob{\uobsaux} := \pdfNorm{\uobs}{\meanvec, \covmat}, \quad \qFunc{\uobsaux} := \pdfNorm{\uobsaux}{\meanvecq, \covmatq},
\end{equation}
for $\meanvec, \meanvecq \in \reals^N$ and positive-definite matrices $\covmat, \covmatq \in \psdmats^N$.
Partition $\uobsaux$ into a collection of $T$ sets, $\uobsaux_1, ..., \uobsaux_T$, such that
\begin{equation}
    \uobsaux = \begin{bmatrix}
        \uobsaux_1 \\
        \vdots \\
        \uobsaux_T
    \end{bmatrix}.
\end{equation}
Let $\meanvec_t$ and $\meanvecq_t$ be the blocks of $\meanvec$ and $\meanvecq$ corresponding to $\uobsaux_t$.
Similarly let $\covmat_{t}$ and $\covmatq_{t}$ the on-diagonal blocks of $\covmat$ and $\covmatq$ corresponding to $\uobsaux_t$.
Then the prior and approximate posterior marginals over $\uobsaux_t$ are
\begin{equation}
    \Prob{\uobsaux_t} = \pdfNorm{\uobsaux_t}{\meanvec_t, \covmat_{t}}, \quad \qFunc{\uobsaux_t} = \pdfNorm{\uobsaux_t}{\meanvecq_t, \covmatq_t},
\end{equation}
due to the marginalisation property of Gaussians.

Assume that the reconstruction term can be written as a sum over terms specific to each $\uobsaux_t$,
\begin{equation}
    \Func{r}{\meanvecq, \covmatq} = \sum_{t=1}^T \Func{r_t}{\meanvecq_t, \covmatq_t},
\end{equation}
for functions $r_1, ..., r_T$.
This is a useful assumption because it is satisfied for the model class considered in this work.
Under this assumption, the optimal Gaussian approximate posterior density is proportional to
\begin{equation}
    \Prob{\uobsaux} \prod_{t=1}^T \pdfNorm{\approxobs_t}{\uobsaux_t, [\prect_t]^{-1}} \label{eqn:approx-post-model-app}
\end{equation}
for appropriately-sized surrogate observations $\approxobs_1, ..., \approxobs_T$ and positive-definite precision matrices $\prect_1, ..., \prect_T$, which is to say that the optimal approximate posterior is equivalent to the exact posterior under a ``surrogate'' Gaussian observation model whose density factorises across $\uobsaux_1, ..., \uobsaux_T$.

\newcommand{\etaq}{\eta^{\textrm{q}}}

A straightforward way to arrive at this result is via a standard result involving exponential families.
Consider an exponential family prior
\begin{equation}
    \Prob{\uobsaux} = \Func{h}{\uobsaux} \Func{\exp}{\inner{}{\eta}{\Func{\phi}{\uobsaux}} - \Func{A}{\eta}},
\end{equation}
where $h$ is the base measure, $\phi$ the sufficient-statistic function, $A$ the log partition function, and $\eta$ the natural parameters.
and approximate posterior in the same family,
\begin{equation}
    \qFunc{\uobsaux} = \Func{h}{\uobsaux} \Func{\exp}{\inner{}{\etaq}{\Func{\phi}{\uobsaux}} - \Func{A}{\etaq}},
\end{equation}
which differs from $p$ only in its natural parameters $\etaq$.
Let
\begin{equation}
    \Cond{\uobsaux}{\yvec} \propto \Prob{\uobsaux} \Cond{\yobs}{\uobsaux}
\end{equation}
be the posterior over $\uobsaux$ given observations $\yvec$ under an arbitrary observation model $\Cond{\yobs}{\uobsaux}$.
It is well-known (see e.g. \citet{khan2018fast}) that the $\etaq$ minimising $\DivKL{\qFunc{\uobsaux}}{\Cond{\uobsaux}{\yobs}}$ satisfies
\begin{equation}
    \etaq = \eta + \Func{(\nabla_{\mu} r)}{\Func{\mu}{\etaq}}. \label{eqn:opt-nat-params}
\end{equation}
where $\mu$ denotes the expectation parameters $\mu := \Expect{q}{\Func{\phi}{\uobsaux}}$ and, in an abuse of notation, $\mu(\eta)$ denotes the function computing the mean parameter for any particular natural parameter.
Given the canonical parameters $\meanvec$ and $\covmat$ of a Gaussian, and letting $\Lambda := [\covmat]^{-1}$, its natural parameters and mean parameters are
\begin{equation}
    \eta = (\eta_1, \eta_2) = (\Lambda \meanvec, -\frac{1}{2} \Lambda), \quad \mu = (\mu_1, \mu_2) = (\meanvec, \meanvec \transpose{\meanvec} + \covmat).
\end{equation}

Let $\precq := [\covmatq]^{-1}$ be the precision of $q$, $\Lambda := \inv{\covmat}$ the precision of $p$, and recall that the optimal Gaussian approximate posterior satisfies
\begin{equation}
    \precq = \Lambda - 2 \, \Func{(\nabla_{\covmatq} r)}{\meanvecq, \covmatq}.
\end{equation}
Due to the assumed structure in $r$, $\nabla_{\covmatq} r$ is block-diagonal:
\begin{equation}
    \Func{(\nabla_{\covmatq} r)}{\meanvecq, \covmatq} = \begin{bmatrix}
        \Func{(\nabla_{\covmatq_{1}} r_1)}{\meanvecq_1, \covmatq_1} & & \mathbf{0} \\
         & \ddots & \\
        \mathbf{0} & & \Func{(\nabla_{\covmatq_{T}} r_T)}{\meanvecq_T, \covmatq_T}
    \end{bmatrix}.
\end{equation}
Observe that each on-diagonal block involves only the corresponding term in $r$, i.e. the $t^{th}$ block is only a function of $r_t$.

Equating the exact posterior precision under the approximate model in \cref{eqn:approx-post-model-app} with the optimal approximate posterior precision yields
\begin{equation}
    \Lambda + \begin{bmatrix}
        \approxprec_1 & & \mathbf{0} \\
        & \ddots & \\
        \mathbf{0} & & \approxprec_T
    \end{bmatrix} =
    \Lambda + \begin{bmatrix}
        -2 \, \Func{(\nabla_{\covmatq_{1}} r_1)}{\meanvecq_1, \covmatq_1} & & \mathbf{0} \\
         & \ddots & \\
        \mathbf{0} & & -2 \, \Func{(\nabla_{\covmatq_{T}} r_T)}{\meanvecq_T, \covmatq_T}
    \end{bmatrix}
\end{equation}
From the above we deduce that letting $\approxprec_t := -2\,\Func{(\nabla_{\covmatq_{t}} r_t)}{\meanvecq_t, \covmatq_t}$ ensures that the posterior precision under the surrogate model and the precision of the approximate posterior coincide for the optimal $q$.

Similarly, the optimal approximate posterior mean satisfies
\begin{equation}
    \precq \meanvecq = \Lambda \meanvec + [\nabla_{\mu} r]_1,
\end{equation}
where $[\nabla_{\mu} r]_1$ denotes the component of the gradient of $r$ w.r.t. $\mu$ corresponding to $\mu_1$.
Equating the optimal posterior mean under the approximate model in \cref{eqn:approx-post-model-app} with that of the optimal approximate posterior yields
\begin{equation}
    \begin{bmatrix}
        \prect_1 & & \mathbf{0} \\
        & \ddots & \\
        \mathbf{0} & & \prect_T
    \end{bmatrix} \begin{bmatrix}
        \approxobs_1 \\
        \vdots \\
        \approxobs_T
    \end{bmatrix} =
    \begin{bmatrix}
        [\nabla_{\mu} r_1]_{1} \\
        \vdots \\
        [\nabla_{\mu} r_T]_{1}
    \end{bmatrix}.
\end{equation}
This implies that $\approxobs_t := \prect_t^{-1} [\nabla_{\mu} r]_{1t}$.

\subsection{Approximate Inference via Exact Inference}
\label{sec:computing-stuff}

Recall the standard saturated bound introduced by \citet{titsias2009variational}, that is obtained at the optimal approximate posterior:
\begin{equation}
    \mathcal{L} = \log \pdfNorm{\yobs}{\meanvec_{\fobs}, \covmat_{\fobs \uobsaux} \Lambda_{\uobsaux} \covmat_{\uobsaux \fobs} + \emissionvar} - \frac{1}{2} \tr{ \inv{\emissionvar} [ \covmat_{\fobs} - \covmat_{\fobs \uobsaux} \Lambda_{\uobsaux} \covmat_{\uobsaux \fobs} ] }. \label{eqn:saturated-bound-appendix}
\end{equation}
The first term is simply to log marginal likelihood of the LGSSM defined in \cref{eqn:lgssm-separable} if $f$ is separable, or \cref{eqn:lgssm-sum-separable} if it is sum-separable.

Recall from \cref{eqn:block-diag-cond} that
\begin{align}
    \covmat_{\fobs \uobsaux} \Lambda_{\uobsaux} =&\, \begin{bmatrix}
        \Bmat_1 & & \mathbf{0} \\
        & \ddots & \\
        \mathbf{0} & & \Bmat_T
    \end{bmatrix}, \\
    \Bmat_t :=&\, \covmat_{\fobs_t \uobs_t} \Lambda_{\uobs_t} \emission_{\Mpertime D}, \nonumber
\end{align}
so the trace term in \cref{eqn:saturated-bound-appendix} can be written as
\begin{align}
    \tr{\inv{\emissionvar} [ \covmat_\fobs- \covmat_{\fobs, \uobsaux} \Lambda_{\uobsaux} \covmat_{\uobsaux, \fobs}  ]} &= \tr{\inv{\emissionvar} [ \covmat_\fobs- \covmat_{\fobs, \uobsaux} \Lambda_{\uobsaux} \covmat_{\uobsaux} \Lambda_{\uobsaux} \covmat_{\uobsaux, \fobs}  ]} \nonumber \\
    &= \sum_{t=1}^T \tr{\inv{\emissionvar_t} [ \covmat_{\fobs_t} - \Bmat_{t} \covmat_{\uobsaux_t} \Bmat_t  ]}  \label{eqn:approx-model-stats} \\
    &= \sum_{t=1}^T \tr{\inv{\emissionvar_t} [ \covmat_{\fobs_t} - \covmat_{\fobs_t \uobs_t} \Lambda_{\uobs_t} \covmat_{\uobs_t} \Lambda_{\uobs_t} \covmat_{\uobs_t \fobs_t} ]} \nonumber \\
    &= \sum_{t=1}^T \tr{\inv{\emissionvar_t} [ \covmat_{\fobs_t} - \covmat_{\fobs_t \uobs_t} \Lambda_{\uobs_t} \covmat_{\uobs_t \fobs_t} ]} \label{eqn:kernel-stats}
\end{align}

These quantities can be computed either by running the approximate model forwards through time and computing the marginal statistics using \cref{eqn:approx-model-stats}, or via the $\kernel^\spacevar$ and $\kernel^\timevar$ directly using \cref{eqn:kernel-stats}.

\newcommand{\fobsnew}[1]{\fobs_{\ast #1}}

Observe that, as with any $\fobs_t$ from the training data, the marginal distribution over some $\fobs_{\ast t}$ under the approximate posterior only involves $\uobsaux_t$ as $\fobs_{\ast t} \bigCI \uobsaux_{\backslash t} \mid \uobsaux_t$:
\begin{align}
    \qFunc{\fobsnew{t}} &= \pdfNorm{\fobsnew{t}}{ \meanq_{\fobsnew{t}}, \covmatq_{\fobsnew{t}} } \text{ where } \nonumber \\
    \meanq_{\fobsnew{t}} &:=\, \meanvec_{\fobsnew{t}} + \covmat_{\fobsnew{t} \uobs_t} \Lambda_{\uobs_t} \emission_{\uobs_t} ( \meanq_{\uobsaux_t} - \meanvec_{\uobsaux_t} ), \nonumber \\
    \covmatq_{\fobsnew{t}} &:=\, \covmat_{\fobsnew{t}} - \covmat_{\fobsnew{t} \uobs_t} \Lambda_{\uobs_t} \emission_{\uobs_t} \left[ \covmat_{\uobsaux_t} - [\precopt_{\uobsaux_t}]^{-1} \right] \transpose{\emission_{\uobs_t}} \Lambda_{\uobs_t} \covmat_{\uobs_t \fobsnew{t}}. \nonumber
\end{align}
Performing smoothing in the approximate model provides $\meanq_{\uobsaux_t}$ and $[\precopt_{\uobsaux_t}]^{-1}$, from which the optimal approximate posterior marginals are straightforwardly obtained via the above.

\section{ FITC }
\label{sec:fitc}

Consider the approximate model employed by FITC:
\begin{equation}
    \CondProbApprox{\yobs}{\uobsaux} := \pdfNorm{\yobs}{\meanvec_{\fobs} + \covmat_{\fobs \uobsaux} \Lambda_{\uobsaux} (\uobsaux - \meanvec_{\uobsaux}), \Func{\text{diag}}{\covmat_{\fobs} - \covmat_{\fobs \uobsaux} \Lambda_{\uobsaux} \covmat_{\uobsaux \fobs}} + \emissionvar }
\end{equation}
We know that in our separable setting, $\covmat_{\fobs \uobsaux} \Lambda_{\uobsaux}$ is block-diagonal from \cref{sec:block-diag-structure}. This means that the $t^{th}$ block on the diagonal of the conditional covariance matrix is
\begin{equation}
    \covmat_{\fobs_t} - \covmat_{\fobs_t \uobsaux_t} \Lambda_{\uobsaux_t} \covmat_{\uobsaux_t \fobs_t},
\end{equation}
and the entire conditional distribution factorises as follows:
\begin{equation}
    \CondProbApprox{\yobs}{\uobsaux} = \prod_{t=1}^T \pdfNorm{\yobs_t}{\meanvec_{\fobs_t} + \covmat_{\fobs_t \uobsaux_t} \Lambda_{\uobsaux_t} (\uobsaux_t - \meanvec_{\uobsaux_t}), \Func{\text{diag}}{\covmat_{\fobs_t} - \covmat_{\fobs_t \uobsaux_t} \Lambda_{\uobsaux_t} \covmat_{\uobsaux_t \fobs_t}} + \emissionvar_t}.
\end{equation}
By comparing this with equation 5 of \citep{hartikainen2011sparse}, and letting the observation model in that equation $\Cond{\yobs_k}{\vec{x}_k} = \pdfNorm{\yobs_k}{[ \Ident_N \kron \emission ] \vec{x}_k, \emissionvar_t}$, the correspondence is clear.

\section{ Inference Under Non-Gaussian Observation Models }
\label{sec:non-gaussian-likelihoods}

While the optimal approximate posterior over the pseudo-points is not Gaussian, in line with most other approximations (e.g. \citep{hensman2015scalable}) we restrict it to be so. As we have shown that the optimal approximate posterior precision is block-tridiagonal regardless the observation model, it follows that the optimal Gaussian approximation must be a Gauss-Markov model. While in general such a model has a total of $T (D \Mpertime + 2(D \Mpertime)^2)$ free (variational) parameters, in our case we know that the off-diagonal blocks of the precision are the same as in the prior, meaning that there are at most $T (D \Mpertime + (D \Mpertime)^2)$ free (variational) parameters -- this is also clear from \cref{eqn:optimal-approx-post-precision}. While one could directly parametrise the precision, this might be inconvenient from the perspective of numerical stability and implementation (standard filtering / smoothing algorithms do not work directly with the precision). Consequently, it probably makes sense to set up a surrogate model in line with that discussed by \cite{khan2017conjugate}, \cite{chang2020fast}, and \cite{ashman2020sparse}. Alternatively one could parametrise the filtering distributions directly, from which the posterior marginals could be obtained using standard smoothing algorithms.

\section{Additional Experiment Details}

\subsection{ Benchmarking Experiment }
\label{sec:additional-experimental-details}

The kernel of the GP used in all experiments is
\begin{equation}
    \Func{\kernel}{(\spacevar, \timevar), (\spaceprime, \timeprime)} = \Func{\kernel^\spacevar}{\spacevar, \spaceprime} \Func{\kernel^\timevar}{\timevar, \timeprime} \label{eqn:separable-kernel-app}
\end{equation}
where $\kernel^\spacevar$ is an Exponentiated Quadratic kernel with length scale $0.9$ and amplitude $0.92$, and $\kernel^\timevar$ is a Matern-3/2 kernel with length scale $1.2$. The particular values of the length scales / amplitudes are of little importance to the proof-of-concept experiments presented in this work -- they were chosen pseudo-randomly.

\begin{figure}[!htbp]
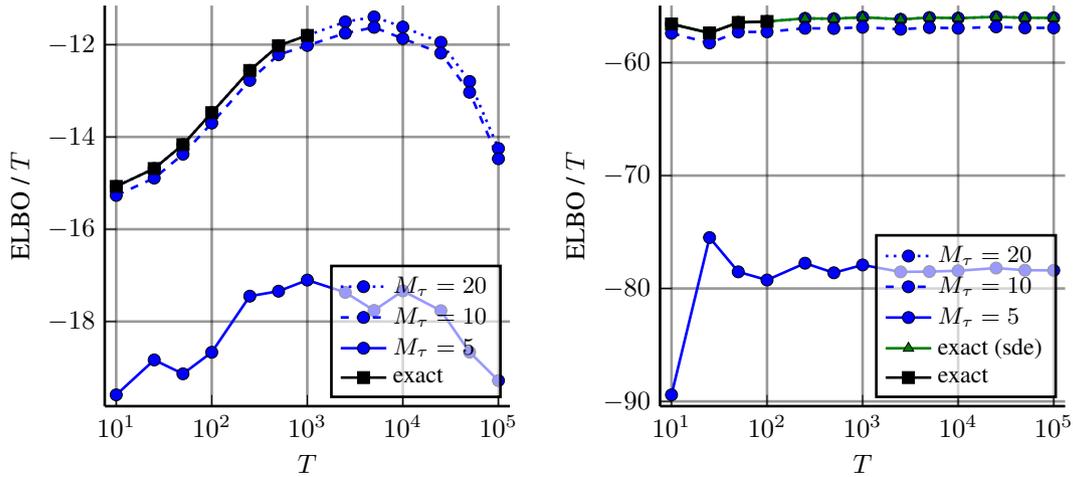

    \centering
    \includegraphics{figures/irregular_lml_plot.tikz}
    \includegraphics{figures/regular_lml_plot.tikz}
    \caption{\label{fig:lml-plot}The ELBO obtained vs the exact LML. The bound appears reasonably tight when $\Mpertime=10$ are used per time point, and very tight for $\Mpertime=20$. $\Mpertime=5$ is clearly insufficient.   }
\end{figure}

These experiments were conducted using a single thread on a 2019 MacBook Pro with 2.6~GHz CPU. Timings produced using benchmarking functionality provided by \citet{benchmarktools}.

\subsubsection{ Sum-Separable Experiments }
\label{sec:sum-separable}

Similar experiments to those in section \cref{sec:experiments} were performed with the sum-separable kernel given by adding two separable kernels of the form in \cref{eqn:separable-kernel}, although with similar length-scales and amplitudes. The results are broadly similar, although the state space approximations and state space + pseudo-point approximations take a bit longer to run as there are twice as many latent dimensions for a given number of pseudo-points than in the separable model. As before, these experiments should be thought of purely as a proof of concept.

\begin{figure}
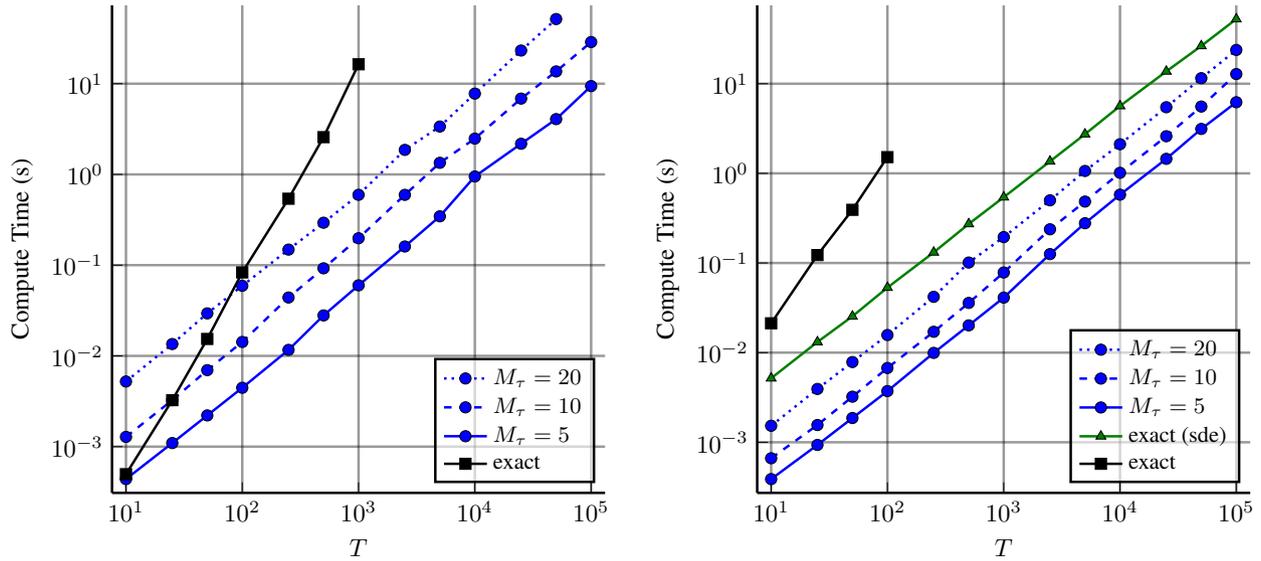

    \centering
    \includegraphics[width=0.49\textwidth]{figures/irregular_timing_plot_sum_separable.tikz}
    \includegraphics[width=0.49\textwidth]{figures/regular_timing_plot_sum_separable.tikz}
    \caption{\label{fig:sum-sep-timings}Time to compute LML exactly vs ELBO with a sum of two separable kernels. Left: irregular samples as per \cref{fig:irregular-plot}. Right: regular samples with missing data as per \cref{fig:rectilinear-plot}. Observe that, due to the increased latent dimensionality of the sum-separable model, it takes longer to compute the ELBO (and LML using the vanilla state space approximation) than in the separable case.}
\end{figure}

\begin{figure}
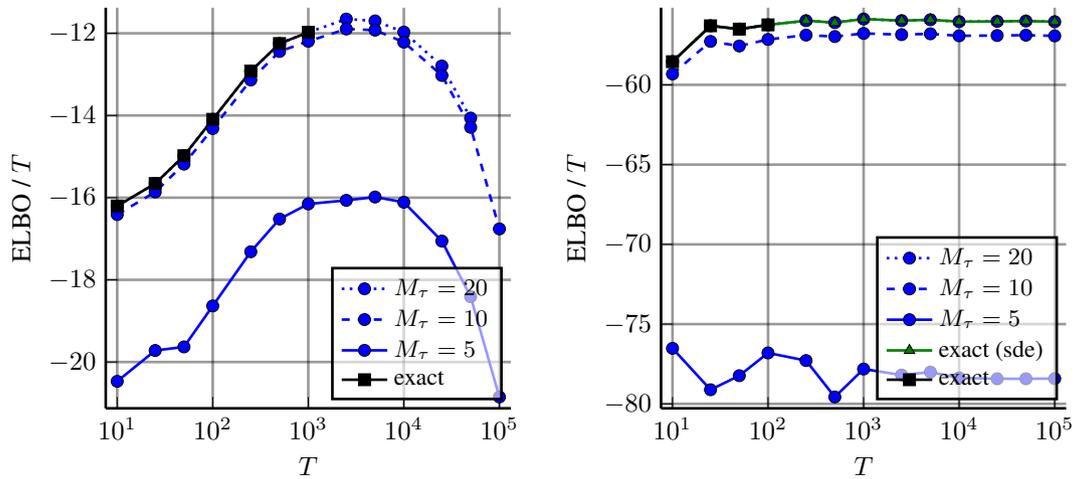

    \centering
    \includegraphics{figures/irregular_lml_plot_sum_separable.tikz}
    \includegraphics{figures/regular_lml_plot_sum_separable.tikz}
    \caption{\label{fig:sum-sep-lml}Analogue of \cref{fig:lml-plot} for \cref{fig:sum-sep-timings}. As before, $\Mpertime=5$ is clearly insufficient for accurate inference, while $\Mpertime=20$ is very close to the LML.}
\end{figure}

\subsection{Climatology Data}
\label{sec:climatology-data-extras}

The spatial locations of the pseudo-points were chosen via k-means clustering of lat-lon coordinates of sold apartments, using Clustering.jl, and were not optimised beyond that.\footnote{\url{https://github.com/JuliaStats/Clustering.jl}}

The separable kernel was
\begin{equation}
    \Func{\kernel}{(\spacevar, \timevar), (\spacevar^\prime, \timevar^\prime)} = s \, \Func{\kernel^\spacevar}{\Lambda \spacevar, \Lambda \spacevar^\prime} \Func{\kernel^{\timevar}}{\lambda \timevar, \lambda \timevar^\prime}
\end{equation}
where $\kernel_\timevar$ is a standardised Mat\'{e}rn-$\frac{5}{2}$, $\kernel_\spacevar$ is a standardised Exponentiated Quadratic, $\Lambda$ is a diagonal matrix with positive elements, $\lambda > 0$, $s > 0$. Initialisation: $\lambda = 10^{-2}$, $\Lambda_{d, d} = 1$, $s=1$, $d \in \{1, 2, 3\}$. Observation noise variance initialised to $0.5$.

The L-BFGS implementation provided by \citet{mogensen2018optim} was utilised to optimise the ELBO, with memory $M=50$ iterations, and gradients computed using the Zygote.jl algorithmic differentiation tool \citep{zygote}. All kernel parameters constrained to be positive by optimising the log of their value, and observation noise variance constrained to be in $[10^{-2}, 2]$ via a re-scaled logit transformation -- this is justifiable as the data itself was standardised to have unit variance.

The sum-separable model comprises a sum of two GPs with kernels of this form. Initialisation: $\lambda = \{10^{-3}, 10^{-1}\}$, $\Lambda_{d, d} = \{ 1.0, 5.0 \}$, $s = \{ 0.7, 0.3 \}$. The same optimisation procedure was used.

\subsection{Apartment Data}
\label{sec:apartment-data-extras}

The spatial locations of the pseudo-points were chosen via k-means clustering of lat-lon coordinates of sold apartments, using Clustering.jl, and were not optimised beyond that. $\Mpertime=75$ pseudo-points used per time point.

The separable kernel was
\begin{equation}
    \Func{\kernel}{(\spacevar, \timevar), (\spacevar^\prime, \timevar^\prime)} = s \, \Func{\kernel^\spacevar}{\Lambda \spacevar, \Lambda \spacevar^\prime} \Func{\kernel^{\timevar}}{\lambda \timevar, \lambda \timevar^\prime}
\end{equation}
where $\kernel_\timevar$ is a standardised Mat\'{e}rn-$\frac{3}{2}$, $\kernel_\spacevar$ is a standardised Exponentiated Quadratic, $\Lambda$ is a diagonal matrix with positive elements, $\lambda > 0$, $s > 0$. Initialisation: $\lambda = 10^{-2}$, $\Lambda_{d, d} = 1$, $s=1$, $d \in \{1, 2\}$. Observation noise variance initialised to $0.5$.

The L-BFGS implementation provided by \citet{mogensen2018optim} was utilised to optimise the ELBO, with memory $M=50$ iterations, and gradients computed using the Zygote.jl algorithmic differentiation tool \citep{zygote}. All kernel parameters constrained to be positive by optimising the log of their value, and observation noise variance constrained to be in $[10^{-2}, 2]$ via a re-scaled logit transformation -- this is justifiable as the data itself was standardised to have unit variance.

The sum-separable model comprises a sum of two GPs with kernels of this form. Initialisation: $\lambda = \{10^{-3}, 10^{-1}\}$, $\Lambda_{d, d} = \{ 1.0, 5.0 \}$, $s = \{ 0.7, 0.3 \}$. The same optimisation procedure was used.

\section{Efficient Inference in Linear Latent Gaussian Models}
\label{sec:efficient-inference-in-conditional}

Consider the linear-Gaussian model
\begin{align}
    \xvec &\sim \dNorm{\meanx, \covx} \nonumber \\
    \yvec \mid \xvec &\sim \dNorm{\Amat \xinp + \avec, \Qmat} \nonumber
\end{align}
where $\meanx \in \reals^{\dimx}$ and $\avec \in \reals^{\dimy}$ are vectors, $\covx$ is a $\dimx \times \dimx$ positive-definite matrix, $\Qmat$ is a $\dimy \times \dimy$ diagonal positive definite matrix, and $\Amat$ is a $\dimy \times \dimx$ matrix. We need to
\begin{enumerate}
    \item generate samples from the marginal distribution over $\Prob{\yvec}$,
    \item compute the marginals $\Prob{\yvec_n}$, $n = 1, ..., \dimy$,
    \item compute the LML $\Prob{\yvec}$, and
    \item compute the posterior distribution $\Cond{\xvec}{\yvec}$.
\end{enumerate}

All of these operations can be performed exactly in polynomial-time since $\xvec$ and $\yvec$ are jointly Gaussian distributed. However, there are two approaches for computing 1, 3, and 4, one of which will be faster depending upon $\dimx$ and $\dimy$.

In this section we analyse these approaches. We do this to prepare for deriving the additional algorithms needed to perform the above operations efficiently when $\Amat := \Bmat \Cmat$, for tall $\Bmat$ and wide $\Cmat$.

\subsection{Preliminaries}

The marginal distibution over $\yvec$ is
\begin{equation}
    \yvec \sim \dNorm{\meany, \covy}, \quad \meany := \Amat \meanx + \avec, \quad \covy := \Amat \covx \transpose{\Amat} + \Qmat \nonumber.
\end{equation}
Computing $\Amat \covx$ requires $\bigO{\dimy \dimx^2}$ operations, and $(\Amat \covx) \transpose{\Amat}$ requires $\bigO{\dimy^2 \dimx}$, so constructing the marginals takes roughly $\bigO{\dimy \dimx^2 + \dimy^2 \dimx}$ operations.

By computing $\Prob{\yvec}$ we will mean computing $\meany$ and $\covy$.

\subsection{Sampling}

First consider sampling from the marginal distribution over $\Prob{\yobs}$. The two approaches to this are:
\begin{enumerate}
    \item Ancestral sampling: first sample from $\Prob{\xvec}$ then from $\Cond{\yvec}{\xvec}$.
    \item Direct marginal sampling: compute the marginal distribution $\Prob{\yvec}$ and sample from it directly.
\end{enumerate}

Ancestral sampling requires computing the Cholesky factorisation of $\covx$, thus the overall algorithm requires $\bigO{\dimx^3 + \dimx \dimy}$ scalar operations. Conversely, computing $\Prob{\yvec}$ and sampling from it requires $\bigO{\dimy \dimx^2 + \dimy^2 \dimx + \dimy^3}$ scalar operations. So if $\dimx$ is much smaller than $\dimy$ we are better off using ancestral sampling, but if $\dimx$ is much larger than $\dimy$ then direct marginal sampling is better.

\subsection{Computing Marginal Probabilities}

To compute all $\Prob{\yvec_n}$ we must compute both $\meany$ and the diagonal of $\covy$. $\meany$ requires only $\bigO{\dimx \dimy}$ scalar operations. Once $\Amat \covx$ has been computed, obtaining the diagonal of $\covy$ requires only an additional $\bigO{\dimx \dimy}$ scalar operations, so the whole operation requires roughly $\bigO{\dimy \dimx^2}$ operations.

\subsection{Computing the Log Marginal Likelihood and Posterior}

These two operations can be performed separately, but it typically makes sense to perform them together as the majority of computational work is shared between them.

\centerline{
\begin{minipage}{0.6\linewidth}
\begin{algorithm}[H]
    \caption{LML by and posterior by factorising $\Prob{\yvec}$. Approx. number of scalar operations on the right of each line.}
    \label{alg:lml-in-y}
    \begin{algorithmic}[1]
      \STATE {\textbf{Naive-Inference}:} $\meanx$, $\covx$, $\Amat$, $\avec$, $\Qmat$, $\yvec$
      \STATE $\Vmat \gets \Amat \covx$ \hfill $\bigO{\dimy \dimx^2}$
      \STATE $\covy \gets \Vmat \transpose{\Amat} + \Qmat$ \hfill $\bigO{\dimy^2 \dimx}$
      \STATE $\Umat \gets \Func{\text{cholesky}}{\covy}$ \hfill $\bigO{\dimy^3}$
      \STATE $\Bmat \gets \invtranspose{\Umat} \Vmat$ \hfill $\bigO{\dimy^2 \dimx}$
      \STATE $\alpha \gets \invtranspose{\Umat} ( \yvec - (\Amat \meanx + \avec) )$ \hfill $\bigO{\dimy^2}$
      \STATE $\text{lml} \gets -\frac{1}{2} \left[ \dimy \log 2\pi + 2 \log \det \Umat + \transpose{\alpha} \alpha \right]$ \hfill $\bigO{\dimy}$
      \STATE $\meanxpost \gets \meanx + \transpose{\Bmat} \alpha$ \hfill $\bigO{\dimy \dimx}$
      \STATE $\covxpost \gets \covx + \transpose{\Bmat} \Bmat$ \hfill $\bigO{\dimx^2 \dimy}$
      \STATE \textbf{return} $\meanxpost, \covxpost, \text{lml}$
    \end{algorithmic}
\end{algorithm}
\end{minipage}
}

\newcommand{\UX}{\Umat_{\xvec}}
\newcommand{\UQ}{\Umat_\Qmat}
\newcommand{\Fmat}{\mathbf{F}}
\newcommand{\Gmat}{\mathbf{G}}

\centerline{
\begin{minipage}{0.8\linewidth}
\begin{algorithm}[H]
    \caption{LML and posterior by exploiting the matrix inversion and determinant lemmas. Approx. number of scalar operations on the right of each line. Note that since $\Qmat$ is diagonal, its Cholesky factorisation is also diagonal.}
    \label{alg:lml-via-inversion-lemma}
    \begin{algorithmic}[1]
      \STATE {\textbf{Low-Rank-Inference}:} $\meanx$, $\covx$, $\Amat$, $\avec$, $\Qmat$, $\yvec$
      \STATE $\UQ \gets \Func{\text{cholesky}}{\Qmat}$ \hfill $\bigO{\dimy}$
      \STATE $\UX \gets \Func{\text{cholesky}}{\covx}$ \hfill $\bigO{\dimx}$
      \STATE $\Bmat \gets \UX \transpose{\Amat} \inv{\UQ}$ \hfill $\bigO{\dimx^2 \dimy}$
      \STATE $\Umat \gets \Func{\text{cholesky}}{\Bmat \transpose{\Bmat} + \Ident}$ \hfill $\bigO{\dimx^3 + \dimx^2 \dimy}$
      \STATE $\Gmat \gets \invtranspose{\Umat} \UX$ \hfill $\bigO{\dimx^3}$
      \STATE $\covxpost \gets \transpose{\Gmat} \Gmat$ \hfill $\bigO{\dimx^3}$
      \STATE $\delta \gets \invtranspose{\UQ} ( \yvec - (\Amat \meanx + \avec) )$ \hfill $\bigO{\dimy}$
      \STATE $\beta \gets \Bmat \delta$ \hfill $\bigO{\dimy \dimx}$
      \STATE $\meanxpost \gets \meanx + \transpose{\Gmat} ( \invtranspose{\Umat} \beta )$ \hfill $\bigO{\dimx^2}$
      \STATE $\text{lml} \gets -\frac{1}{2} \left[ \transpose{\delta} \delta - \transpose{(\invtranspose{\Umat} \beta)} \invtranspose{\Umat} \beta + \dimy \log 2\pi + 2 \log \det \Umat + 2 \log \det \Qmat \right]$ \hfill $\bigO{\dimx^2}$
      \STATE \textbf{return} $\meanxpost, \covxpost, \text{lml}$
    \end{algorithmic}
\end{algorithm}
\end{minipage}
}

Note that \cref{alg:lml-in-y} and \cref{alg:lml-via-inversion-lemma} are locally scoped, so the symbols don't necessarily correspond to the same quantities. For example, $\Bmat$ is different in each algorithm.

\subsection{Bottleneck Linear-Gaussian Observation Models}

The above inference methods assume no particular structure in $\Amat$, however, recall that \cref{eqn:lgssm-separable} gives $\Amat$ at time $t$ to be
\begin{equation}
    \Amat = \covmat_{\fobs_{n, t} \uobs_{t}} \Lambda_{\uobs_t} \emission_{\uobs}
\end{equation}
where $\emission_{\uobs}$ is $M \times MD$ and $\covmat_{\fobs_{n, t} \uobs_{t}}$ is $N \times M$. Most kernels have $D > 1$, so it's worth determining whether we can exploit this structure to accelerate inference.

\newcommand{\Hmat}{\mathbf{H}}
\newcommand{\hvec}{\mathbf{h}}
\newcommand{\bvec}{\mathbf{b}}

To this end consider a model given by
\begin{align}
    \zvec :=&\, \Hmat \xvec + \hvec, \quad \xvec \sim \dNorm{\meanx, \covx} \\
    \yvec :=&\, \Bmat \zvec + \bvec + \varepsilon, \quad \varepsilon \sim \dNorm{\mathbf{0}, \Qmat}
\end{align}
where $\Hmat \in \reals^{M \times DM}$, $\hvec \in \reals^{M}$, $\Bmat \in \reals^{N \times M}$, $\bvec, \varepsilon \in \reals^{N}$, and $\Qmat \in \reals^{N \times N}$ is a positive-definite diagonal matrix.
We call this a \textit{bottleneck} model, since $\zvec$ carries all of the information in $\xvec$ needed to perform inference in $\yvec$, its dimension is less than that of $\xvec$ and $\yvec$ in the problems that we consider.

Observe that this model forms a two-state degenerate Markov chain.
\cref{alg:bottleneck-lml-posterior} exploits this Markov structure, and is able to recycle \cref{alg:lml-via-inversion-lemma} as a consequence.
It comprises three broad components: computing the marginals over $\zvec$, computing the posterior $\zvec | \yvec$, and finally computing the posterior $\xvec | \yvec$.
This last step is equivalent to performing a single step of RTS smoothing.

\newcommand{\meanz}{\meanvec_{\zvec}}
\newcommand{\covz}{\covmat_{\zvec}}
\newcommand{\meanzpost}{\meanvec_{\zvec | \yvec}}
\newcommand{\covzpost}{\covmat_{\zvec | \yvec}}
\newcommand{\dimz}{D_z}

\centerline{
\begin{minipage}{0.8\linewidth}
\begin{algorithm}[H]
    \caption{LML and posterior. Exploits the matrix inversion and determinant lemmas, and the bottlenecked structure of the model. Approx. number of scalar operations on the right of each line.}
    \label{alg:bottleneck-lml-posterior}
    \begin{algorithmic}[1]
        \STATE {\textbf{Bottleneck-Inference}:} $\meanx$, $\covx$, $\Bmat$, $\bvec$, $\Qmat$, $\Hmat$, $\hvec$, $\yvec$
        \STATE $\meanz \gets \Hmat \meanx + \hvec$ \hfill $\bigO{ \dimz \dimx }$
        \STATE $\covz \gets \Hmat \covx \transpose{\Hmat} + \hvec$ \hfill $\bigO{ \dimz \dimx^2 + \dimz^2 \dimx }$
        \STATE $\meanzpost, \covzpost, \text{lml} \gets \Func{\textbf{Low-Rank-Inference}}{\meanz, \covz, \Bmat, \bvec, \Qmat}$ \hfill $\bigO{ \dimz^2 \dimy + \dimz^3 }$
        \STATE $\Umat \gets \Func{\text{cholesky}}{\covzpost}$ \hfill $\bigO{\dimz^3}$
        \STATE $\Gmat \gets \covx \transpose{\Hmat} \inv{\Umat} \invtranspose{\Umat}$ \hfill $\bigO{\dimx^2 \dimz}$
        \STATE $\meanxpost \gets \meanx + \transpose{\Gmat} ( \meanzpost - \meanz )$ \hfill $\bigO{ \dimz \dimx }$
        \STATE $\covxpost \gets \covx + \transpose{\Gmat} (\covzpost - \covz) \Gmat$ \hfill $\bigO{ \dimz^2 \dimx + \dimx^2 }$
        \STATE \textbf{return} $\meanxpost, \covxpost, \text{lml}$
    \end{algorithmic}
\end{algorithm}
\end{minipage}
}

Observe that this algorithm exchanges $\bigO{\dimx^3}$ for $\bigO{\dimz^3}$ operations.
Recalling that $\dimx = MD$ and $\dimz = M$, we expect that this algorithm will produce better performance than \cref{alg:lml-via-inversion-lemma} for some value of $D > 1$. 
Since the exact value of $D$ at which this change will occur is unclear, and the optimal choice of algorithm for the experiments in this work depends on this, we investigate the effect of $D$, $M$, and $N$ on the performance of each algorithm in the next subsection.

\subsection{Benchmarking Inference}

\begin{figure}[!htbp]
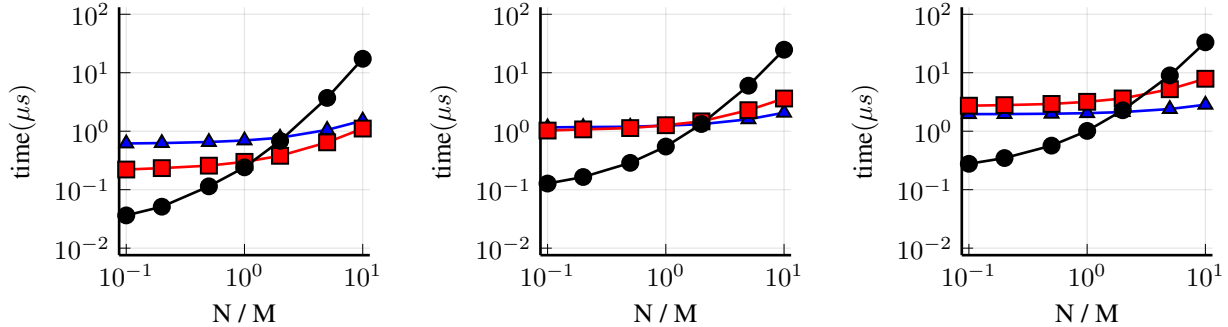

    \centering
    \includegraphics{figures/emission-benchmarks//D=1.tikz}
    \includegraphics{figures/emission-benchmarks//D=2.tikz}
    \includegraphics{figures/emission-benchmarks//D=3.tikz}
    \caption{\label{fig:emission-benchmarks}Black-circle=naive, red square=low rank, blue triangle=bottleneck. All experiments conducted using $M=100$ pseudo-points. Left: $D=1$, Middle: $D=2$, Right: $D=3$.}
\end{figure}

\cref{fig:emission-benchmarks} shows the performance of the three different algorithms for computing the LML and posterior distribution as the dimensions of the linear-Gaussian model's dimensions change.
Black lines with circles use \cref{alg:lml-in-y} (labelled \textit{naive}), red lines with squares used \cref{alg:lml-via-inversion-lemma} (labelled \textit{low rank}), and blue lines with triangles use \cref{alg:bottleneck-lml-posterior} (labelled \textit{bottleneck}).
The experiments are set up to match situations encountered in the spatio-temporal models discussed in this paper -- they are parametrised in terms of the number of $M$, $N$, and $D$. $M$ is fixed to $100$ across all three graphs, the total number of latent dimensions is $MD$ for $D \in \{ 1, 2, 3 \}$, corresponding to the Mat\'{e}rn-$1/2$, Mat\'{e}rn-$3/2$, and Mat\'{e}rn-$5/2$ kernels respectively.
The number of observations $N$ range between $0.1M = 10$ and $10M = 1000$.

As expected the \textit{naive} algorithm performs better when $N < M$, but this quickly changes when $N > M$.
For the kinds of problems encountered in this work we generally have that $N > M$, which would suggest that correct choice in our work is typically the \textit{low-rank} algorithm.
For $M=N$ the \textit{naive} algorithm tends to be faster, owing to the smaller number of operations used -- it's just a shorter algorithm than the \textit{low-rank} algorithm.

\cref{alg:bottleneck-lml-posterior}, \textit{bottleneck}, performs similarly or better than \cref{alg:lml-via-inversion-lemma} for $D=2$ and $D=3$.
The gap between the two grows as $N$ grows, suggesting that \cref{alg:bottleneck-lml-posterior} will typically be a better choice for large $M$. Indeed, even for $D=1$ the difference between the two becomes close for large $N$.

Given that $N$, $M$ and $D$ determine which of the three algorithms is optimal, one must choose appropriately for any given application.
We adopt the bottleneck algorithm in all experiments in this work because we consistently work in regimes where $N > M$ at most points in time, and we do not make use of any kernels for which $D=1$.

Also note that these results highlight that using \cref{alg:lml-via-inversion-lemma} to perform the second step in \cref{alg:bottleneck-lml-posterior} is only optimal if $N > M$.
If a problem were encountered for which $N < M$ it would be prudent to consider replacing it with \cref{alg:lml-in-y}.

\end{document}